\documentclass[twoside,11pt]{article}

\usepackage{blindtext}
\usepackage{placeins}
\usepackage{amsmath}

\usepackage{listings}
\usepackage{xcolor}

\usepackage{float}
\floatstyle{ruled}          
\newfloat{pseudocode}{tbp}{lop}
\floatname{pseudocode}{Pseudocode}

\usepackage[normalem]{ulem}

\newcommand{\changed}[1]{#1}

\usepackage{longtable} 
\usepackage{geometry}  
\usepackage{jmlr2e}

\usepackage{graphicx}
\usepackage{bbm}

\usepackage{float}
\usepackage{booktabs}       
\usepackage{amsfonts}       
\usepackage{nicefrac}       
\usepackage{microtype}      
\usepackage{xcolor}         
\usepackage{bbm}
\usepackage{booktabs}
\usepackage{multirow}
\usepackage{graphicx}
\usepackage{tabularx}
\usepackage{algorithm}
\usepackage{makecell}
\usepackage{enumitem}
\allowdisplaybreaks

\usepackage{algpseudocode}

\providecommand{\abs}[1]{\lvert#1\rvert}
\providecommand{\norm}[1]{\lVert#1\rVert}

\usepackage{lastpage}
\jmlrheading{26}{2025}{1-\pageref{LastPage}}{1/25; Revised 6/25}{8/25}{25-0134}{Shihao Shao, Yikang Li, Zhouchen Lin, and Qinghua Cui}

\newcommand{\bracedincludegraphics}[2][]{%
  \sbox0{$\vcenter{\hbox{\includegraphics[#1]{#2}}}$}%
  \left\lbrace
    \vphantom{\copy0}
  \right.\kern-\nulldelimiterspace
  \underbrace{\box0}}

\ShortHeadings{High-Rank ICT Decomposition and Bases of Equivariant Spaces}{Shao, Li, Lin, and Cui}
\firstpageno{1}

\begin{document}
\title{High-Rank Irreducible Cartesian Tensor Decomposition and Bases of Equivariant Spaces}

\author{\name Shihao Shao$^\dagger$\email shaoshihao@pku.edu.cn \\
       \addr State Key Lab of Vascular Homeostasis and Remodeling\\
       \addr Department of Bioinformatics\\
       \addr School of Basic Medical Sciences, Peking University
\AND
    \name Yikang Li \email liyk18@pku.edu.cn \\
       \addr State Key Lab for General AI\\
       \addr School of Intelligence Science and Technology \\
       \addr Peking University, China
\AND
    \name Zhouchen Lin$^\dagger$\email zlin@pku.edu.cn \\
       \addr State Key Lab for General AI\\
       \addr School of Intelligence Science and Technology, Peking University\\
       \addr Pazhou Laboratory (Huangpu), China
\AND
    \name Qinghua Cui$^\dagger$\email cuiqinghua@bjmu.edu.cn \\
       \addr State Key Lab of Vascular Homeostasis and Remodeling\\
       \addr Department of Bioinformatics\\
       \addr School of Basic Medical Sciences, Peking University
       }

\footnotetext[2]{To whom correspondence should be addressed}

\editor{Michael Mahoney}

\maketitle

\begin{abstract}
Irreducible Cartesian tensors (ICTs) play a crucial role in the design of equivariant graph neural networks, as well as in theoretical chemistry and chemical physics. Meanwhile, the design space of available linear operations on tensors that preserve symmetry presents a significant challenge. The ICT decomposition and a basis of this equivariant space are difficult to obtain for high-rank tensors. After decades of research, \cite{bonvicini2024irreducible} has recently achieved an explicit ICT decomposition for $n=5$ with factorial time/space complexity. In this work we, for the first time, obtain decomposition matrices for ICTs up to rank $n=9$ with reduced and affordable complexity, by constructing what we call path matrices. The path matrices are obtained via performing chain-like contractions with Clebsch-Gordan matrices following the parentage scheme. We prove and leverage that the concatenation of path matrices is an orthonormal change-of-basis matrix between the Cartesian tensor product space and the spherical direct sum spaces. Furthermore, we identify a complete orthogonal basis for the equivariant space, rather than a spanning set \citep{pearce2023brauer}, through this path matrices technique. \changed{Our method avoids the RREF algorithm and maintains a fully analytical derivation of each ICT decomposition matrix, thereby significantly improving the algorithm’s speed to obtain arbitrary rank orthogonal ICT decomposition matrices and orthogonal equivariant bases.} We further extend our result to the arbitrary tensor product and direct sum spaces, enabling free design between different spaces while keeping symmetry. The Python code is available at https://github.com/ShihaoShao-GH/ICT-decomposition-and-equivariant-bases, where the $n=6,\dots,9$ ICT decomposition matrices are obtained in 1s, 3s, 11s, and 4m32s on 28-core Intel\textsuperscript{\textregistered} Xeon\textsuperscript{\textregistered} Gold 6330 CPU @ 2.00GHz, respectively.
\end{abstract}


\section{Introduction}
\label{sec:introduction}

The deployment of equivariant neural networks in physics, chemistry, and robotics often requires certain equivariance or invariance to effectively incorporate inductive bias. \changed{One typical equivariance requirement is equivariance under the $O(3)$ group (rotations and reflections).} Equivariant Graph Neural Networks (EGNNs) are well-suited for this purpose, as they can maintain equivariance while processing features within the graph \citep{zaverkin2024higher,simeon2024tensornet,batzner20223}. One of the most promising approaches to constructing equivariant layers involves leveraging irreducible representations. Depending on the basis of the space, two distinct types of irreducible representations are utilized in different EGNNs: spherical (spherical tensors) \citep{musaelian2023learning,batatia2022mace} and Cartesian irreducible representations (Irreducible Cartesian Tensors, ICTs) \citep{zaverkin2024higher,schutt2021equivariant} for $O(3)$. Spherical tensors serve as the starting point for EGNNs based on irreducible representations. However, they require Clebsch-Gordan transforms with significant computational overhead to derive desired irreducible representations from existing ones during each inference, making them computationally intensive. In contrast, ICTs are more efficient, as high-rank representations can be easily computed using tensor products (specifically, outer products in EGNNs) and contractions, making them more computationally friendly for most deep learning frameworks. The ICTs are obtained through the ICT decomposition matrices $H$,

\begin{equation}
\label{eq:tht}
    {\rm vec}\left(\hat{T}^{(n;l;q)}\right) = H^{(n;l;q)} {\rm vec}\left(T^{(n)}\right),
\end{equation}

\noindent where $\hat{T}^{(n;l;q)}$ is an ICT, ${\rm vec}$ is the vectorization operation via flattening the tensor, $T^{(n)}$ is a given tensor, $n$ represents the rank of the tensor, and $l$ is called the weight (the order of irreducible representations, which is distinct from the learnable weights in neural networks) of the irreducible representation, and $q$ denotes the label (since the multiplicity for given $n$ and $l$ is sometimes greater than $1$). \changed{In this paper, we adopt the convention that $\operatorname{vec}$ flattens a tensor by stacking its entries from the rightmost to the leftmost index; concretely, ${\rm vec}(T) = \{T_{...111},T_{...112},\cdots\}$.} The decomposition matrices $H^{(n;l;q)}$ should be constructed such that $\hat{T}^{(n;l;q)}$ are irreducible. A formal definition is introduced in Section \ref{sec:ictdecomp}. The challenge lies in the fact that $H^{(n;l;q)}$ of rank $n>5$ are unknown. The recent development of EGNNs research benefits from the irreducible decomposition matrices $H$ obtained from the physics and chemistry communities. The decomposition matrices for $n=4$ have been established by \cite{andrews1982irreducible}. However, only very recently has \cite{bonvicini2024irreducible} obtained the decomposition matrices for $n=5$. 
\changed{Due to the lack of high-rank ICT decomposition matrices, it is impossible to construct high-rank Cartesian EGNNs that incorporate valid equivariant linear layers. That is because only irreducible representations of same weight can be linearly combined, but the irreducible representations under Cartesian basis are mixed together. In other words, the irreducible representations of the same weight are impossible to be separated, so we cannot linearly combine those of the same type, which limits the available rank of irreducible representations in EGNNs with full access to the equivariant linear layer designs. The fact that utilizing high-rank irreducible representations improves performance has been well investigated by predecessors, with the support of several empirical evidence. For example, see Table 2 in the paper by \cite{batzner20223}, where rank $l=3$ NequIP significantly outperforms the same model with lower ranks. On the other hand, in the work done by \cite{liao2023equiformerv2}, Table 1(c) shows that the performance is better with higher ranks.} 
Therefore, the matrices for high-rank orthogonal ICT decomposition are of great importance. Based on this, we can obtain valid equivariant linear operations applied to the Cartesian tensors, as these operations constitute the design space of EGNNs. In other words, what is the vector space of linear mappings $f$ such that

\begin{equation}
\label{eq:equivbasisdef}
    \rho(g)f\left({\rm vec}(T)\right)=f\left(\rho(g){\rm vec}(T)\right)
\end{equation}

\noindent where $g \in G$, \changed{an element of a group $G$. $T\in \left(\mathbbm{R}^3\right)^{\otimes n}$ is an arbitrary rank-$n$ tensor. $\rho:G\to GL\left(\left(\mathbbm{R}^3\right)^{\otimes n}\right)$ is the group representation, and $GL(V)$ is a space consisting of the invertible matrices that operate on vector space $V$. A detailed introduction to the basic and necessary group theory is in Section \ref{sec:irreduciblerepandsphericalspace}}. This vector space is called an equivariant space, also denoted as ${\rm End}_{G}\left(\left(\mathbbm{R}^3\right)^{\otimes n}\right)$. More generally, for different input and output spaces, this equivariant space can be denoted as ${\rm Hom}_{G}\left(\left(\mathbbm{R}^3\right)^{\otimes l},\left(\mathbbm{R}^3\right)^{\otimes k}\right)$. \cite{pearce2023brauer} leverages \cite{brauer1937algebras}'s theorem to find the spanning set of the equivariant space. However, as the rank increases, the number of elements in the spanning set becomes greater than the dimension of the space. Thus, it is also essential to identify a basis to reduce the number of parameters required.

In this work, one of our contributions is the proposal of a method that does not rely on RREF algorithm, for efficiently obtaining ICT decomposition matrices, through which we achieve $6 \leq n \leq 9$ in an affordable setting. Several previous methods \citep{andrews1982irreducible, dincckal2013orthonormal, bonvicini2024irreducible} rely on evaluating the full permutation of $n$ (e.g., fundamental isotropic Cartesian tensors) and $2n$ (e.g., natural projection tensors) indices, which requires factorial complexity, making it impractical for application to high-rank tensors. \cite{bonvicini2024irreducible} proposes a novel method to automate the procedure but does not reduce the complexity. The Reduced Row Echelon Form (RREF) algorithm is also required to run on an $n! \times n!$ matrix. Instead, we leverage the fact that the Cartesian tensor product space is equivalent to a direct product of spherical spaces, up to a change of basis. In spherical space, it is straightforward to identify the desired irreducible representations, as they are diagonalized; then, the inverse change-of-basis can bring us back to Cartesian space. The challenge lies in finding such a change-of-basis matrix. We prove that one can construct this matrix by concatenating the results of sequential contractions and the outer products of Clebsch-Gordan matrices (CG matrices) according to the parentage scheme \citep{coope1965irreducible}, which we refer to as path matrices. 
\changed{The orthogonality of the path matrices further allows us to eliminate the time-consuming inverse operation, rendering this approach an orthogonal decomposition. Both theoretical and experimental analyses confirm the efficiency of this approach.}
\changed{Inspired by Schur’s lemma, which states that only irreducible representations of the same weight can be linearly combined, we construct an equivariant basis via a path-matrix technique. This construction makes it possible to design high-rank equivariant linear layers in EGNNs. We then address the lack of a suitable parentage scheme for general spaces. To this end, we propose a generalized parentage scheme that yields an equivariant basis for arbitrary spaces. These advances collectively enable the efficient construction of equivariant layers across different spaces, thereby expanding the design space of EGNNs.}


\begin{table}[]
\centering
\setlength{\tabcolsep}{5pt}
\renewcommand{\arraystretch}{1.5}

\begin{tabular}{lcc|cc}
\hline
\multirow{2}{*}{Method} & \multirow{2}{*}{Rank} & \multirow{2}{*}{Orthogonality} & \multicolumn{2}{c}{Time and Space Complexity} \\ \cline{4-5}
& & & w.r.t. rank & w.r.t. dimension \\ \hline
\cite{coope1965irreducible}    & $\leq3$ & No  & \multicolumn{2}{c}{calculated by hand} \\
\cite{andrews1982irreducible}  & $\leq4$ & No  & \multicolumn{2}{c}{calculated by hand} \\
\cite{dincckal2013orthonormal} & $\leq3$ & Yes & \multicolumn{2}{c}{calculated by hand} \\

\cite{bonvicini2024irreducible}& $\leq5$ & No  & 
   \multicolumn{2}{c}{
       \begin{tabular}{@{}c@{\,/\,}c@{}}
           factorial & sub-exponential
       \end{tabular}
   } \\

Ours                           & $\leq9$ & Yes & 
   \multicolumn{2}{c}{
       \begin{tabular}{@{}c@{\,/\,}c@{}}
           exponential & polynomial
       \end{tabular}
   } \\ \hline

\end{tabular}
\caption{Comparison with previous works. The time and space complexities are with respect to rank and dimension, respectively.}
\label{tab:rwcompare}
\end{table}


\vspace{0.2cm}

In summary, this work has five main contributions:

\begin{itemize}
\item We propose a method to efficiently obtain high-rank ICT decomposition matrices and, for the first time, obtain the decomposition matrices for $6 \leq n \leq 9$, significantly improving upon the current limit of $n = 5$, as shown in Table \ref{tab:rwcompare}.
\item We propose a general parentage scheme to extend ICT decomposition into a generalized ICT decomposition applicable to arbitrary tensor product spaces.
\item We propose to directly obtain orthogonal bases of equivariant spaces, whereas the method proposed by \cite{pearce2023brauer} yields a spanning set, requiring an expensive post-processing step to extract a basis that is not orthogonal in general.
\item \changed{Our method avoids the RREF algorithm and derives each ICT decomposition matrix in a fully analytical manner, thereby yielding a significantly more efficient algorithm.}
\item We extend our approach to efficiently obtain bases of equivariant spaces mapping between different spaces, including spherical tensor product spaces.
\end{itemize}

These contributions enable EGNNs to process high-rank ICTs and flexibly design equivariant layers between desired spaces. Furthermore, this work benefits the physics and chemistry communities by facilitating theoretical analysis based on high-rank ICTs.

\newpage

\begin{longtable}{|m{1.6cm}|m{5cm}|m{1.6cm}|m{5cm}|} 
\caption{Mathematical symbols and their meanings} \label{tab:symbols}\\ 
\hline
\textbf{Symbol} & \textbf{Meaning} & \textbf{Symbol} & \textbf{Meaning} \\ \hline
\endhead
\hline
\endfoot

\hline
\(\norm{\cdot}_{col}\) & The $L_2$-norm of an arbitrary column of $\cdot$ & \(G\) & An arbitrary group  \\ \hline
\(SO(n)\) & Special orthogonal group on $n$-dimensional space & \(SU(n)\) & Special unitary group on $n$-dimensional space \\ \hline
\(O(n)\) & Orthogonal group on $n$-dimensional space & \(U(n)\) & Unitary group on $n$-dimensional space \\ \hline
$\mathcal{V}$ & A vector space & \(\rho_{\mathcal{V}}\) & A group representation\\ \hline
$\mathcal{V}^{\otimes n}$ & A rank $n$ tensor product space from $\mathcal{V}$ & \(\changed{N^{(n;l)}}\) & The multiplicity of the ICT of rank $n$ and \changed{weight $l$}\\ \hline
\(T\) & An arbitrary tensor \(G \to \mathcal{V}\) & \(T^{(n)}\) & A Cartesian tensor of rank \(n\) \\ \hline
\(\changed{\hat{T}^{(n;l;q;p)}}\) & An ICT of rank \(n\), \changed{weight \(l\)}, index \(q\), and parity $p$ (sometimes omitted) & $Q^{(l;p)}_{\mathcal{V}}$ & The set of all possible paths to decompose a tensor in space $\mathcal{V}$ \\ \hline
\(\changed{H^{(n;l;q)}}\) & An ICT decomposition matrix to obtain ICTs of rank \(n\), \changed{weight \(\l\)}, and index \(q\) & $\changed{\hat{H}^{(\mathcal{V};l;q)}}$ & A general ICT decomposition matrix of space $\mathcal{V}$, \changed{weight $l$} and index $q$  \\ \hline
\(C^{l_1l_2l_o}_{m_1m_2m_o}\) & \multicolumn{3}{|m{12.6cm}|}{A CG coefficient with input weights \(l_1\), \(l_2\), output weight \(l_o\), and the corresponding indices \(m_1\), \(m_2\), and \(m_o\) (the angular-momentum quantum numbers in physics)} \\ \hline
$C^{l_1l_2l_o}$ & A CG tensor with input weights $l_1$, $l_2$, and output weight $l_o$ & $\hat{C}^{l_1l_2l_o}$ & A CG matrix by merging the first two dimensions of the original CG tensor\\ \hline
$P^{(path)}$ & A path matrix generated by $path$, e.g., $(1\to 3\to 2)$ & $\hat{P}^{(path)}$ & A normalized path matrix generated by $path$\\ \hline
$D^{(\mathcal{V}\to \mathcal{V}^\prime)}$ & A change-of-basis matrix from space $\mathcal{V}$ to $\mathcal{V}^\prime$ & $v^{(l;p)}$ & A vector on which weight $l$ and parity $p$ irreducible representation acts\\ \hline
${\rm Hom}_G$ & The set of $G$-equivariant maps from one space to another & ${\rm End}_G$ & The set of $G$-equivariant maps from one space to itself\\ \hline
irr & Returns all weights of irreducible representations in a decomposition & ${\rm dim}$ & The number of dimensions of a space  \\ \hline

$(i_1\overset{k_1}{\to}i_2 $ $\overset{k_2}{\to} \dots)$ & A path where $k_*$ are bridge numbers (omitted if all being $1$) & \( (l,p)\) & An irreducible representation space of weight $l$ and parity $p$ \\ \hline
\(\otimes\) & Tensor product & \(\oplus\) & Direct sum \\ \hline
\(\odot\) & Contraction & $\#$ & Number of elements in a set \\ \hline

\end{longtable}
\newpage

The following contents are organized as follows. \changed{In Section \ref{sec:rep}, we recall basic concepts from representation theory, spherical harmonics, tensors, and tensor product spaces, and we introduce the CG coefficients along with the parentage scheme.} In Section \ref{sec:problem}, we formulate our problems: obtaining high-rank ICT decomposition matrices and bases for arbitrary tensor product spaces. In Section \ref{sec:highorderict}, we introduce our theory and method to obtain high-rank ICT decomposition matrices. In Section \ref{sec:basis}, we present the theory and method to obtain bases of equivariant spaces where the input and output spaces are the same. In Section \ref{sec:extension}, we extend the theory to the equivariant design spaces where the input and output spaces are different, and show the method to obtain bases for such general spaces. In Section \ref{sec:relatedworks}, we review related works on designing EGNNs and calculating ICT decomposition matrices. In Section \ref{sec:discussion}, we discuss the relation between \cite{pearce2023brauer} and this work, as well as the theory for $SO(n)$, $O(n)$, and $SU(n)$. We provide concluding remarks in Section \ref{sec:conclusion}, and introduce the reproducibility details in Section \ref{sec:repro}. The acknowledgments are presented after the main text. \changed{The proofs of supportive lemmas are shown in Appendix \ref{sec:proofs}.} The complexity comparison with \cite{bonvicini2024irreducible} is in Appendix \ref{sec:complexityanalysis}. \changed{We give an example for generating rank-2 ICT decomposition
matrices in Appendix \ref{sec:examplerank2}.} The visualizations of ICT decomposition matrices and equivariant bases are shown in Appendix \ref{sec:visualization}. A list of mathematical symbols is provided in Table \ref{tab:symbols}, and a mind map of this work is shown in Figure \ref{fig:mindmap}.

\section{Preliminaries}
\label{sec:rep}
In this section, we recap necessary mathematical preliminaries.

\subsection{Group, Irreducible Representations and Spherical Spaces}
\label{sec:irreduciblerepandsphericalspace}

\changed{We first recall some basic facts of group theory}. A thorough introduction to group theory can be found in \cite{inui2012group} and \cite{brocker2013representations}. \changed{A group $G$ is a set equipped with a binary operation $*$, such that the following five properties hold. 
\begin{itemize}  
  \item \textbf{Non-emptiness}:\\
        The set \(G\) is non-empty.
  \item \textbf{Closure}:\\
        For all \(g_1, g_2 \in G\), we have $g_1 * g_2 \in G$.        
  \item \textbf{Associativity}:\\
        For all \(g_1, g_2, g_3 \in G\), we have $(g_1 * g_2) * g_3 = g_1 * (g_2 * g_3)$.
  \item \textbf{Identity element}:\\
        There exists \(e \in G\) such that, for every \(g \in G\), it holds that $e * g = g * e = g$.
  \item \textbf{Inverse element}:\\
        For every \(g \in G\) there exists \(g^{-1} \in G\) satisfying $g * g^{-1} = g^{-1} * g = e$.
\end{itemize}
Group is a powerful tool to study symmetry. Rotations and rotations with reflections are both examples of groups. For example, it is easy to verify that 2-d rotations satisfy the above requirements: clockwise rotating by $30^{\circ}$, followed by counterclockwise $15^{\circ}$, is equal to clockwise rotating by $15^{\circ}$, and the remained properties can be examined similarly. Until now, the group remains abstract, and we cannot bring it into computation. Representation theory provides us with a bridge to connect group to linear algebra. A representation is a mapping $\rho:G\to GL(V)$ such that 
\begin{equation*}
    \rho(g_1*g_2) = \rho(g_1)\cdot\rho(g_2)
\end{equation*}
holds for any $g_1,g_2\in G$, where $GL(V)$ is a space consisting of the invertible matrices that operate on vector space $V$. In other words, they are invertible matrices that operate on vector space $V$. Here, we introduce groups that we will consider in this work. The orthogonal group, $O(n)$, consisting of orthogonal matrices in $\mathbbm{R}^{n\times n}$, represents $n$-dimensional rotations and reflections; The special orthogonal group, $SO(n)$, consisting of orthogonal matrices with determinant equaling $1$, represents $n$-dimensional rotations; The unitary group, $U(n)$, consisting of unitary complex matrices in $\mathbbm{C}^{n\times n}$, represents $n$-dimensional rotations and reflections on $\mathbbm{C}^{n\times n}$, and $SU(n)$ with additional unit determinant constraint, represents rotations only.}  Consider a vector $v \in \mathcal{V}$. 
A group element $g$ acting on the vector $v$ can therefore be represented as a matrix-vector multiplication $\rho(g) \cdot v$. Sometimes, one may also refer to the space of $v$ as a representation space. If the results of the group action remain in the same space, i.e., $\rho(g)v \in \mathcal{V}$, this space is also called a $G$-invariant space or subspace. 
It is important to note that the co-domain of $\rho$ can differ, meaning that it can also be represented by matrices of other dimensions. In this paper, we denote the co-domain of $\rho$ using a subscript when it should be made explicit, e.g., $\rho_{\mathbbm{R}^3}$. Now, it suffices to define equivariance. We claim that a linear function $f: \mathcal{V} \to \mathcal{V}^\prime$ is equivariant if and only if

\begin{equation*}
    \rho_{\mathcal{V}^\prime}f(x)=f(\rho_{\mathcal{V}}(x)).
\end{equation*}
We first focus on the simpler case where $\mathcal{V} = \mathcal{V}^\prime$. We will later generalize this to the case where $\mathcal{V} \neq \mathcal{V}^\prime$. When it is necessary to emphasize the change in space, we refer to a linear mapping $f$ as $G$-$(\mathcal{V}, \mathcal{V}^\prime)$-equivariant. We then introduce the definition of reducible and irreducible representation.

\definition{ \label{def:reducible}
    A group representation $\rho_{\mathcal{V}}: G\to GL(\mathcal{V})$ is said to be reducible, if there is a non-zero subspace $\mathcal{V}^\prime\subset \mathcal{V}$, for which we have a representation $\rho_{\mathcal{V}^\prime}$, such that for every $g\in G$ and $v\in \mathcal{V}^\prime$, it holds that
    \begin{equation*}
        \rho_{\mathcal{V}^\prime}(g)v\in \mathcal{V}^\prime,
    \end{equation*}
    where we also call $\mathcal{V}^\prime$ a $G$-invariant subspace. Conversely, if there is no such subspace $\mathcal{V}^\prime$ and the corresponding representation $\rho_{\mathcal{V}^\prime}$ on it for a given space $\mathcal{V}$ satisfying the above conditions, then we claim that $\rho_{\mathcal{V}}$ is irreducible. }
\vspace{0.2cm}

\noindent If a representation is irreducible, then the vectors space it acts on is also said to be irreducible. The $3\times 3$ rotation and reflection matrices are irreducible. More generally,

\proposition{\label{prop:naturalrepisirreducible}
The natural representation $\{ M\in GL(n, \mathbbm{R}) \mid MM^{\top}=M^\top M=I\}$ of $O(n)$ is irreducible.
}
\vspace{0.2cm}

 Irreducible representations play a very important role in theoretical chemistry and physics, as they represent different symmetries. It is also one of the major building blocks for EGNNs in recent years. Note that the matrix representation will change if we change the basis of the space $\mathcal{V}$ it acts on. Thus, the representation will be different according to the basis we choose. In this paper, we first focus on the $O(3)$ group, but we also extend our conclusions to $O(n)$, $SO(n)$, $U(n)$, and $SU(n)$ in Section \ref{sec:othergroups}. For the spherical basis, each irreducible representations space is spanned by the \textbf{real} spherical harmonics $Y_{lm}$ as a basis of different degree $l$ and order $m$. The complex spherical harmonics $Y_l^m: S^2\to \mathbbm{C}$ are the solutions of Laplace's equation, which are in the form of associated Legendre polynomials $P_l^m$ taking spherical coordinates $(\theta, \varphi)$ as variables (or equivalently $(x,y,z)$ after some change-of-basis). $l$ can take values from $\{\frac{1}{2},1,\frac{3}{2},\dots\}$. They are highly related to $SU(2)$ group, which is of great importance for quantum physics. Leveraging the fact that $SU(2)$ is a double-cover of $SO(3)$, analysis implies that the real form of spherical harmonics $Y_{lm}$ of integer $l$ form a mapping from $S^2$ to $S^{2l}$, via $(x,y,z)\mapsto (Y_{l,-l}(x,y,z),\cdots,Y_{l,l}(x,y,z))$ that the weight-$l$ $O(3)$ irreducible representation (the term \textit{weights} in this paper refer to the order of irreducible representations, following the convention) operate on. In this case, $l\in\{1,2,\dots\}$ and $m\in\{-l,\dots,l\}$, and accordingly $O(3)$ has $3\times 3$, $5\times 5$, $\dots$ matrices as representations. We sometimes use $(l=k)$ to denote a weight-$k$ spherical space. As said above, an important property of spherical harmonics is:

\proposition{
\label{prop:rep}
    Given a space $\mathcal{V}$ spanned by the real spherical harmonics,
    \begin{equation*}
    \left\{(Y_{l,-l}(v),\cdots,Y_{l,l}(v) )\in S^{2l}\;\colon\; v\in S^2\right\},
    \end{equation*}
    it holds that, for any $v\in \mathcal{V}$ and $g\in O(3)$, we always have
    $\rho_{S^{2l}}(g)v\in \mathcal{V}$.
    }
    
\vspace{0.2cm}

The above proposition directly follows from the fact that the mapping $(Y_{l,-l},\cdots,Y_{l,l})$ is $O(3)$-equivariant. There is also a connection between spherical and Cartesian spaces, which is one of the building blocks of our arguments:

\lemma{
\label{lemma:firstdegree}
    The weight one real spherical space is exactly the 3-d Cartesian space up to an orthogonal change-of-basis.}
\vspace{0.2cm}

 This lemma is often utilized in some engineering implementations \citep{thomas2018tfn,weiler20183d,kondor2018clebsch}. In this work, it is a key connecting the spherical and Cartesian space. Also, we have the following properties:

\proposition{The $l$-degree real spherical harmonics are homogeneous harmonic polynomials, $Y_{lm}\in H_l(S^2)$, where $Y_{lm}$ is orthogonal to $Y_{l^\prime m^\prime}$ for $l^\prime \neq l$ or $m^\prime \neq m$, under $L^2(S^2)$. $\{Y_{lm}\}$ forms an orthonormal basis for $(2l+1)$-dimensional vector spaces, with unit L$_2$-norms.}

\subsection{Clebsch-Gordan Coefficients}

  An important tool utilized in our work is the Clebsch-Gordan coefficients (CG coefficients). They are real numbers, and describe how to couple vectors from two spherical spaces into a new one. There are some rules for this coupling. First, given two vectors in weights $l_1$ and $l_2$ spherical spaces, the weights of the resulting vectors can only lie in $|l_1-l_2|\leq l_o\leq l_1+l_2$. For each index $m_o$ of the obtained vector $v_{l_o}\in (l=l_o)$, it is calculated by homogeneous quadratic polynomials. Specifically, we have

\begin{equation}
\label{eq:CG}
    v^{(l_o)}_{m_o} = \sum_{m_o=m_1+ m_2} C^{l_1l_2l_o}_{m_1 m_2 m_o} v_{m_1}^{(l_1)} v_{m_2}^{(l_2)},
\end{equation}
and if we consider $O(n)$ instead of $SO(n)$, parity $p$ should also be considered. The parity is a property additional to the rank and weight, being $1$ or $-1$. Having $-1$ parity means that the representation will be inverse when the space is inverse, and \textit{vice versa}. For example, the natural representations of $O(n)$ always operate on $p=-1$ vector spaces. Taking the parity into consideration, we can guarantee $O(n)$ equivariance, and equation (\ref{eq:CG}) becomes

\begin{equation*}
\label{eq:CGo3}
    v^{(l_o;p_o)}_{m_o} = \mathbbm{1}_{(p_o=p_1p_2)}\sum_{m_o =m_1+m_2} C^{l_1l_2l_o}_{m_1 m_2 m_o} v_{m_1}^{(l_1;p_1)} v_{m_2}^{(l_2;p_2)},
\end{equation*}
where $\mathbbm{1}_{(exp)}$ is the indicator function, outputting $1$ if the expression $exp$ holds true, or $0$ otherwise. Our discussion will be the same no matter we consider $O(n)$ or $SO(n)$. But in and after Section \ref{sec:extension}, we need to take care of the parity which will play some roles in the further discussion. The CG coefficients are highly sparse, since $C^{l_1l_2l_o}_{m_1m_2m_o}=0$ if $m_o\neq m_1+m_2$. The coefficients are identical to Wigner-3j symbols up to multiplicative factors, so we will adopt Wigner-3j for the implementation. The evaluation of CG coefficients for $SU(2)$ is nearly $\mathcal{O}(1)$ (if we do not consider the few calculations for the closed-form formula, which is cheap enough). To obtain the coefficients of $O(3)$ or $SO(3)$, one often leverages the relation between $SU(2)$ and $SO(3)$: $SU(2)$ is a double cover of $SO(3)$, so the CG matrices of $SU(2)$ and $SO(3)$ are the same up to a change-of-basis, from complex to real. The closed-form formulae for $SU(2)$ CG coefficients are already known \citep{theoryofcomplexspectra}. Thus, this process can be efficiently evaluated without obvious overhead. 

\begin{figure*}[tb]

  \centering
  \includegraphics[width=1.0\linewidth]{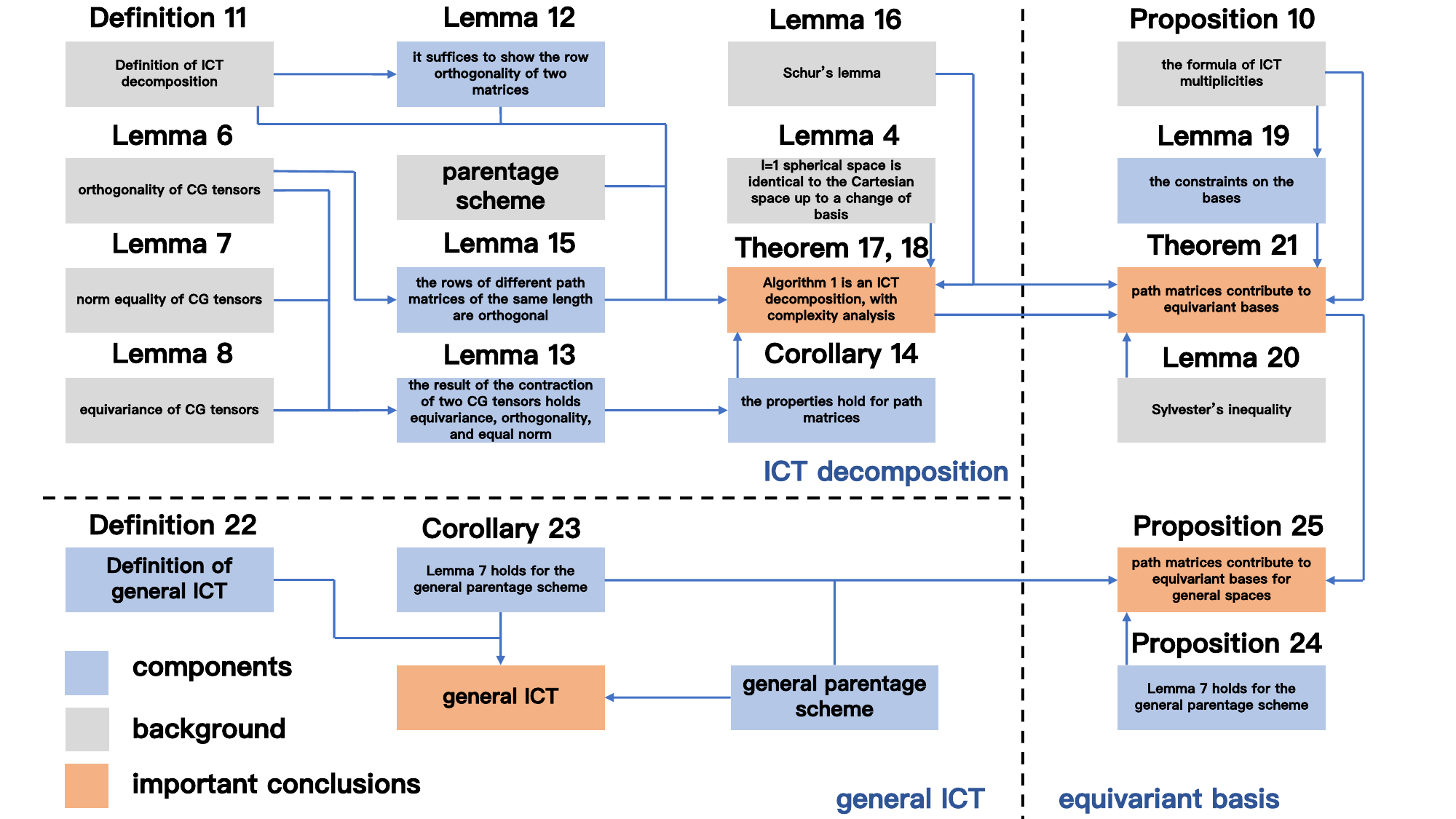}
  
  \caption{The mind map of our method. The blocks in orange color are aligned with the contributions of this work. \changed{The ICT decomposition algorithm (Theorem \ref{theorem:ict}), together with Schur’s lemma (Lemma \ref{lemma:schur}), shows that path matrices can also be used to construct an equivariant basis. Building on the fact that our ICT decomposition relies on the parentage scheme, we propose a generalized parentage scheme that allows decomposition in arbitrary spaces and thus yields equivariant bases for those spaces.}
  }
  \label{fig:mindmap}
\vspace{-0.2cm}
\end{figure*}

\subsection{Tensors and Tensor Product Spaces}
To formally define a tensor, we need to consider vector space $V$ and covector space $V^{*}$. A tensor product is a bilinear map $\otimes:\mathcal{V}\times \mathcal{V}^\prime\to \mathcal{V}\otimes \mathcal{V}^\prime$. The basis of $\mathcal{V}\otimes \mathcal{V}^\prime$ is the set of $\hat{x}_i\otimes \hat{y}_j\in \mathcal{V}\otimes \mathcal{V}^\prime$, where $\hat{x}_i$ and $\hat{y}_j$ are basis vectors of $\mathcal{V}$ and $\mathcal{V}^\prime$, respectively. That is why one implements outer product as tensor product, as it meets the definition of tensor product. A tensor $T:V^{\otimes r}\otimes V^{*\otimes s}\to F$, where $F$ is some field, is a mathematical object that converts things between different tensor product spaces. For example, the representation $\rho(g)$ is an $(r=1,s=1)$-tensor that receives a vector $v$ which is an $(r=0,s=1)$-tensor, and output another $(r=0,s=1)$-tensor $\rho(g)v$, as $r$ and $s$ vanish correspondingly. Tensor product can help us raise the ranks of tensors, e.g., $V\otimes V \otimes V$ is an $(r=0,s=3)$ space. The representation space is accordingly $(r=3,s=3)$. Conversely, contraction can reduce the ranks of tensors. We express these operations via some index notations. For example, $T_{i_1i_2} = T^{\prime}_{i_1}T^{\prime\prime}_{i_2}$ is a tensor product, and $T_{i_1} = \sum_{j_1} T^{\prime}_{i_1j_1}T^{\prime\prime}_{j_1}$ is a contraction. The contractions and tensor products are both $O(n)$-equivariant. Obviously, such operations can be mixed together, and we agree that $i_k$ will be placed at the $k$th index of the output tensor by default. We also agree that $j_1$ and $j_2$ vanish due to the contraction. \changed{In other words, $j_1$ and $j_2$ serve as dummy indices and do not appear in the final expression. As shown in the contraction example above, $j_1$ is summed over and thus disappears from the left-hand side of the equation.} Meanwhile, we will not use Einstein summation convention in this work to keep the conventions of machine learning community, and to avoid ambiguity during derivations. One should note that contractions can be expressed in term of matrix representation. Consider $T_{\changed{i_{1}}}=\sum_{j_1,j_2} T^{\prime}_{i_1j_1j_2}T^{\prime\prime}_{j_1j_2}$, we have

\begin{align*}
    &\begin{bmatrix} T_1 \\ T_2 \end{bmatrix} 
= \begin{bmatrix}
T^{\prime}_{111} & T^{\prime}_{112} & T^{\prime}_{121} & T^{\prime}_{122} \\
T^{\prime}_{211} & T^{\prime}_{212} & T^{\prime}_{221} & T^{\prime}_{222}
\end{bmatrix} \begin{bmatrix} T^{\prime\prime}_{11} \\ T^{\prime\prime}_{12} \\ T^{\prime\prime}_{21} \\ T^{\prime\prime}_{22} \end{bmatrix}.
\end{align*}
It is useful when we need to obtain the inverse of some linear operations. Thus, we sometimes denote $v \in \mathbb{R}^{3^n}$ (instead of $\left(\mathbb{R}^{3}\right)^{\otimes n}$) as a vector, if the only possible operation is tensor contraction with $M \in \mathbb{R}^{3^n} \times \mathbb{R}^{3^n}$, where $M$ can also be viewed as a matrix. However, by its nature, $v$ is an $(r = 0, s = n)$-tensor, and $M$ is an $(r = n, s = n)$-tensor. \changed{For a multi-index tensor $T_{i_1i_2i_3}$, sometimes we need to flatten some of the indices. During this operation, we consistently apply a ``right-to-left'' flattening style, similar to the vectorization $\mathrm{vec}$ in Section~\ref{sec:introduction}. For example, consider $T\in \mathbbm{R}^{3\times 3\times 3}$. If we flatten its first two indices, we obtain a matrix,
\begin{equation*}
\begin{pmatrix}
T_{x x x} & T_{x x y} & T_{x x z} \\
T_{x y x} & T_{x y y} & T_{x y z} \\
T_{x z x} & T_{x z y} & T_{x z z} \\
T_{y x x} & T_{y x y} & T_{y x z} \\
T_{y y x} & T_{y y y} & T_{y y z} \\
T_{y z x} & T_{y z y} & T_{y z z} \\
T_{z x x} & T_{z x y} & T_{z x z} \\
T_{z y x} & T_{z y y} & T_{z y z} \\
T_{z z x} & T_{z z y} & T_{z z z} \\
\end{pmatrix}.
\end{equation*}
}

\subsection{CG Matrices and the Parentage Scheme}

\begin{figure*}[tb]

  \centering
  \includegraphics[width=0.75\linewidth]{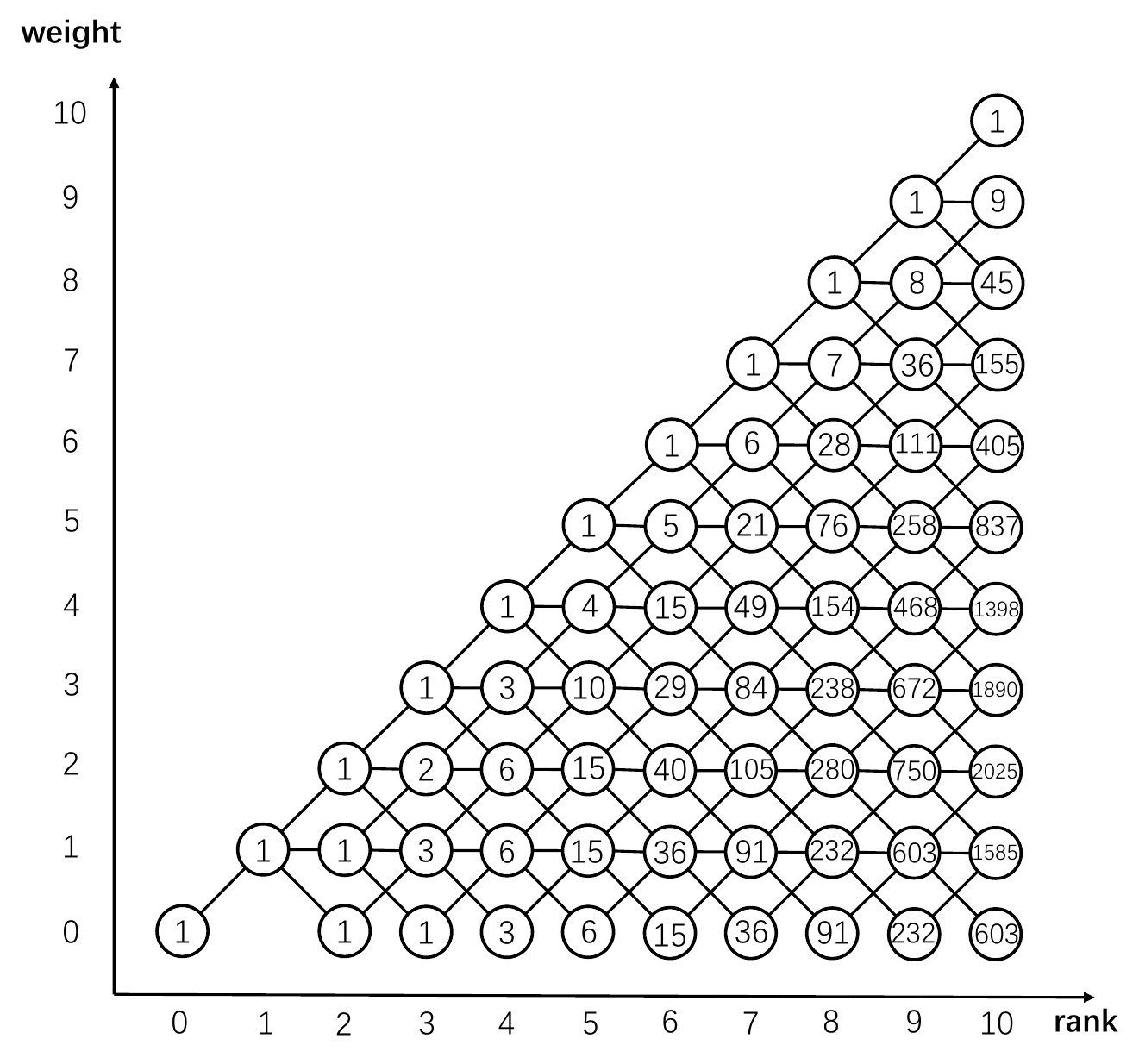}
  \vspace{-0.4cm}
  \caption{The parentage scheme. \changed{For ICTs $\hat{T}^{(n;l;q)}$, the x-axis corresponds to the rank $n$, the y-axis corresponds to the weight $l$, and the numbers in the circles represent the multiplicity $N^{(n;l)}$ of index $q$ in equation (\ref{eq:multiplicity}).} The line connection indicates that the irreducible representation with the weight in the left column can be decomposed into the irreducible representation with the weight in the right column, when coupling with an $l=1$ irreducible representation.
  }
  \label{fig:scheme}
\vspace{-0.4cm}
\end{figure*}

The $O(3)$ group has its representations in tensor product space $(\mathbbm{R}^3)^{\otimes n}$, which are some $3^n\times 3^n$ matrices, or equivalently tensor product of $n$ $(r=1,s=1)$-tensors, which depends on one's view. If we interpret them as matrices, then the objects they act on are simply $3^n$-dimensional vectors. The tensor product spaces $(\mathbbm{R}^3)^{\otimes n}$ for $n>1$ are generally reducible. Recall CG coefficients, two irreducible spherical vectors $v$ and $v^{\prime}$ coupling can also be viewed as that they first form a rank 2 tensor, and then decompose into different new spherical vectors. These operations can also be viewed as a product of matrices and vectors. In this context, the CG coefficients $\hat{C}^{l_1l_2l_o}$ form a $(2l_1+1)(2l_2+1)\times(2l_o+1)$ matrix, which we call a CG matrix. If we reshape it to a $(2l_1+1)\times(2l_2+1)\times(2l_o+1)$ tensor, then we instead call it an $(l_1,l_2,l_o)$-tensor $C^{l_1l_2l_o}$. Here are some important properties.

\lemma[Orthogonality]{\label{lemma:ortho} Given an $(l_1,l_2,l_o)$-CG tensor $C^{l_1l_2l_o}$, we have
\begin{align*}
    \sum_{j_1,j_2}C^{l_1l_2l_o}_{j_1j_2k}C^{l_1l_2l_o}_{j_1j_2k^\prime} = \sum_{j_1,j_2}C^{l_1l_2l_o}_{j_1kj_2}C^{l_1l_2l_o}_{j_1k^\prime j_2} = \sum_{j_1,j_2}C^{l_1l_2l_o}_{kj_1j_2}C^{l_1l_2l_o}_{k^\prime j_1 j_2} = 0
\end{align*}
for $k\neq k^\prime$. More generally, different CG tensors with the same $l_1$ and $l_2$, $l_1$ and $l_o$, or $l_2$ and $l_o$, but differing in other weights, are also orthogonal in this way. Additionally, in such cases, this orthogonality holds true even when $k=k^\prime$.

}

\vspace{0.2cm}

\lemma[Norm]{\label{lemma:norm} Given an $(l_1,l_2,l_o)$-CG tensor $C^{l_1l_2l_o}$, we have
\begin{align*}
    \sum_{j_1,j_2}(C^{l_1l_2l_o}_{j_1j_2k})^2 = \sum_{j_1,j_2}(C^{l_1l_2l_o}_{j_1j_2k^\prime})^2; \quad \sum_{j_1,j_2}(C^{l_1l_2l_o}_{j_1kj_2})^2 = \sum_{j_1,j_2}(C^{l_1l_2l_o}_{j_1k^\prime j_2})^2; \quad \sum_{j_1,j_2}(C^{l_1l_2l_o}_{kj_1j_2})^2 = \sum_{j_1,j_2}(C^{l_1l_2l_o}_{k^\prime j_1j_2})^2
\end{align*}}

\noindent for any pair of $k$ and $k^\prime$.

\vspace{0.2cm}

\lemma[Equivariance]{\label{lemma:equi} An $(l_1, l_2,l_o)$-CG tensor $C^{l_1l_2l_o}$ maintains the following equivariances,
\begin{equation*}
    \rho_{l_1\otimes l_2}(g) \odot C^{l_1l_2l_o} = C^{l_1l_2l_o}\odot\rho_{l_o}(g),
\end{equation*}
where $\odot$ denotes the tensor contraction. This equation holds for other permutations as Lemmas \ref{lemma:ortho} and \ref{lemma:norm}.
}

\vspace{0.2cm}

The above lemmas cover the CG matrices case. Since the columns of CG matrices have the same $L_2$ norm, we denote as $\norm{\cdot}_{col}$ for shorthand, and we sometimes call the last index of the CG tensors as the index of their columns. Recall that CG matrices are non-zero if and only if $|l_1-l_2|\leq l_o\leq l_1+l_2$. If we start from an $l=1$ vector and couple an additional $l=1$ once at a time, then this process has been drawn as a parentage scheme by \cite{coope1965irreducible}, as shown in Figure \ref{fig:scheme}. Each time we move from left to the adjacent right column we couple an additional $l=1$ vector. The parentage scheme also tells us which irreducible representations can a rank $n$ tensor be decomposed into. Note that a rank $n$ tensor can be decomposed into multiple spherical vectors, and the multiplicity for each weight given the rank is written in the scheme. Equivalently, it can also be decomposed into the same numbers of ICTs, but in the Cartesian spaces. We first review the definition of the ICT:

\definition[ICT space]{\label{def:ICT}A rank $n$, weight $l$, and index $q$ ICT space is a subspace, denoted as $\mathcal{V}^{(n;l;q)}\subseteq \left(\mathbbm{R}^3\right)^{\otimes n}$, such that for any $T\in \mathcal{V}^{(n;l;q)}$ and $g\in O(3)$, it holds true that
\begin{equation*}
    \rho_{\left(\mathbbm{R}^3\right)^{\otimes n}}(g)\odot T\in \mathcal{V}^{(n;l;q)},
\end{equation*}
and we also have
\begin{equation*}
    {\rm dim}\ \mathcal{V}^{(n;l;q)}=2l+1.
\end{equation*}
Meanwhile, there is no non-trivial $O(3)$-invariant subspace of $\mathcal{V}^{(n;l;q)}$.
}
\vspace{0.2cm}

Note that there can be multiple rank $n$, weight $l$ ICT spaces with different indices $q$. The tensor product space $\left(\mathbbm{R}^3\right)^{\otimes n}$ can be decomposed into multiple ICT spaces, and different ICT spaces only intersect at $\{0\}$. Given a tensor from the tensor product space, the sum of all the ICT decomposition matrices should equal $I$. \changed{We also have a closed-form formula for the multiplicity, which is obtained by predecessors \citep{momentaMihailov_1977,AndrewsSymmetry},}

\proposition{\label{prop:multip} The multiplicity of the \changed{weight $l$} ICTs of a rank $n$ tensor is 
\begin{align}
\label{eq:multiplicity}
\changed{N^{(n;l)}= \sum\limits^{\lfloor (n-\changed{l})/3 \rfloor}_{i\changed{=0}} (-1)^i\frac{n!(2n-3i-\changed{l}-2)!}{i!(n-i)!(n-2)!(n-3i-\changed{l})!}.}
\end{align}
}

\vspace{0.2cm}

Note that the number scales exponentially with respect to $n$, but it is smaller than $3^n$, the dimension of the tensor product space.

\section{Problem Formulation}
\label{sec:problem}

We first formulate the task of obtaining ICT decomposition matrices. The properties derived during the process of obtaining these decomposition matrices are further leveraged to construct the bases for equivariant linear mappings between tensor product spaces.

\subsection{ICT Decomposition}
\label{sec:ictdecomp}
In this paper, one of our goals is to find an efficient way to obtain a set of matrices $\{H^{(n;l;q)}\}$ to decompose a rank $n$ tensor into ICTs with different weight $l$ and index $q$, as shown in equation (\ref{eq:tht}). From the previous discussion, we know that ICT decomposition matrices remain invariant regardless of the tensor being decomposed, as the decomposition matrices are constructed with respect to the tensor product spaces. Each decomposition matrix should be able to map a tensor from the tensor product space to its corresponding ICT subspace with rank $n$, weight $l$, and index $q$. Given the definition of ICT space (Definition \ref{def:ICT}), we precisely define what an ICT decomposition is. In other words, we specify the properties that the matrices $H$ must satisfy:

\begin{definition}[ICT decomposition]
\label{def:ICTdecomp}
    Given a Cartesian tensor product space $(\mathbbm{R}^{3})^{\otimes n}\cong\mathbbm{R}^{3^n}$ formed by tensor product of Cartesian spaces $\mathbbm{R}^3$, an ICT decomposition is a set of decomposition matrices $H^{(n;l;q)}\in \mathbbm{R}^{3^n\times 3^n}$, where $l$ is the weight of ICTs, and $q$ is the index of the multiplicities, such that:

\vspace{0.2cm}
     1) $H^{(n;l;q)}$ are $O(3)$-$((\mathbbm{R}^{3})^{\otimes n},(\mathbbm{R}^{3})^{\otimes n})$-equivariant.

\vspace{0.2cm}

     2) The summation of the decomposition matrices is an identity matrix, i.e., $\sum_{l,q} H^{(n;l;q)}=I$.

\vspace{0.2cm}

     3) The rank (matrix rank) of the decomposition matrix $H^{(n;l;q)}$ is $2l+1$.

\vspace{0.2cm}

     4) $\{H^{(n;l;q)}v\mid v\in \mathbbm{R}^{3^n}\}$ has no non-trivial $O(3)$-invariant subspace.

\vspace{0.2cm}

     5) $\{H^{(n;l;q)}v\mid v\in\mathbbm{R}^{3^n}\}\bigcap \{H^{(n;l^\prime;q^\prime)}v\mid v\in\mathbbm{R}^{3^n}\}=\{0\}$ holds for $l\neq l^\prime$ or $q\neq q^\prime$.

\vspace{0.2cm}

\noindent Furthermore, if we have 

\begin{align}
    \left(H^{(n;l;q)}v\right)\odot\left(H^{(n;l^\prime;q^\prime)}v^\prime\right) =0
\end{align}
for any $v,v^\prime\in \mathbbm{R}^{3^n}$, and $l\neq l^\prime$ or $q\neq q^\prime$, then we further call it an orthogonal ICT decomposition, and the matrices orthogonal decomposition matrices.
\end{definition}

\vspace{0.2cm}

Consider condition 1), for any $g\in O(3)$ and $v\in R^{3^n}$, we must have

\begin{equation*}
    \rho_{\mathbbm{R}^{3^n}}(g)H^{(n;l;q)}v\in \{H^{(n;l;q)}v \mid v\in R^{3^n}\},
\end{equation*}
as
\begin{equation*}
    \rho_{\mathbbm{R}^{3^n}}(g)H^{(n;l;q)}v=H^{(n;l;q)}\rho_{\mathbbm{R}^{3^n}}(g)v,
\end{equation*}
where $\rho_{\mathbbm{R}^{3^n}}(g)v\in R^{3^n}$. Thus, condition 1) makes $\{H^{(n;l;q)}v \mid v\in R^{3^n}\}$ an $O(3)$-invariant subspace. Meanwhile, condition 3) can also guarantee that 

\begin{equation*}
{\rm dim}\ \{H^{(n;l;q)}v \mid v\in R^{3^n}\}=2l+1,
\end{equation*}
and therefore $\{H^{(n;l;q)}v \mid v\in R^{3^n}\}$ is a rank $n$, weight $l$, and index $q$ ICT space. Note that ICT decomposition is not necessary to be orthogonal, but being orthogonal is the ultimate goal for ICT decomposition. \changed{From now on, the proofs of supportive lemmas are shown in the Appendix \ref{sec:proofs} to save space and maintain a clear flow.} To prove the orthogonality, we have

\lemma{\label{lemma:ortholemma} Given matrices $M:\mathcal{V}\to \mathcal{V}^\prime$ and $M^\prime:\mathcal{V}\to \mathcal{V}^\prime$ of the same shape, if $M^\top M^\prime=0$, then 
\begin{equation}
    (M v)\odot(M^\prime v^\prime) = 0
\end{equation}
for any $v,v^\prime\in \mathcal{V}$.
}

\vspace{0.2cm}





This work obtains orthogonal ICT decomposition matrices for tensor product spaces with rank $6\leq n\leq 9$. A comparison with predecessors is shown in Table \ref{tab:rwcompare}. Theoretically, our method can obtain ICT decomposition matrices with arbitrary ranks analytically. While the full set of $n=10$ ICT decomposition matrices can be computed in a reasonable amount of time, storing such a massive number of matrices poses significant challenges due to the required storage space. However, in practical implementation, it is possible to obtain a subset of ICT decomposition matrices with desired weights, rather than generating the entire set of matrices.

\subsection{Basis of the Equivariant Design Space}
\label{sec:problembasis}

Given a tensor $T \in (\mathbbm{R}^3)^{\otimes n}$, a critical challenge in the design of EGNNs is determining which linear operations acting on the tensor are equivariant, i.e., find the set of $f$ in equation (\ref{eq:equivbasisdef}), denoted as ${\rm End}_{O(3)}((\mathbbm{R}^3)^{\otimes n})$. In real implementation, we flatten $T$ to vectors $v$, and the linear mappings $f$ are therefore matrices. \cite{pearce2023brauer} leverages \cite{brauer1937algebras}'s theorem to obtain spanning sets for homomorphism spaces. However, the spanning set becomes increasingly redundant as the rank of tensors rises. In this work, we aim to directly find orthogonal bases, instead of spanning sets, for equivariant spaces.

Additionally, we seek to find bases for the equivariant spaces ${\rm Hom}_{O(3)}(\mathcal{V}, \mathcal{V}^\prime)$, where $\mathcal{V}$ is not necessarily equal to $\mathcal{V}^\prime$. Here, $\mathcal{V}$ and $\mathcal{V}^\prime$ can be arbitrary direct sums of tensor product spherical spaces (which also include Cartesian spaces, recall Lemma \ref{lemma:firstdegree}). This approach enables us to design EGNNs more freely by making the domain and co-domain of equivariant layers steerable. In other words, given an arbitrary domain and its co-domain, one can efficiently construct equivariant layers without redundant parameters. A discussion of the relationship between \cite{pearce2023brauer} and this work is provided in Section~\ref{sec:pearce-crump}.

\section{High-Rank Orthogonal ICT Decomposition}
\label{sec:highorderict}

In this section, we aim to develop a method to efficiently obtain the ICT decomposition matrices, i.e., to address the challenge outlined in Section \ref{sec:ictdecomp}. In Section \ref{sec:ictdecomptheory}, we first introduce the theoretical foundation, along with proofs, which form the basis of our method. Then, in Section \ref{sec:ictdecompexample}, we provide an example to demonstrate how our method works.

\subsection{The Theory and Method}
\label{sec:ictdecomptheory}

According to the previous discussion, given a vector $v\in \mathbbm{R}^{3^n}$ (flattened from the tensor), we can decompose it into multiple irreducible vectors $v^{(1)},\dots,v^{(k)}$. Under the Cartesian basis, the vector $v$ equals to the \textbf{sum} of the irreducible vectors, as

\begin{equation*}
    v = \sum_{i} v^{(i)}.
\end{equation*}
However, the things are easier when we consider under the spherical basis. The above relation becomes
\begin{equation*}
    v = {\rm concat}(v^{(1)},\dots,v^{(k)}),
\end{equation*}
i.e., \textbf{concatenation} of vectors. It is easy to select out the desired irreducible vectors in this scenario. Thus, we first convert the Cartesian space to the direct sum of the spherical spaces, then, choose the desired irreducible representation, and finally convert it back to the Cartesian tensor product space. Therefore, the key here is to find a change-of-basis matrix between these two spaces. \changed{But how to construct it? Inspired by the fact that CG matrix is an equivariant mapping from the tensor product of two spaces to a single space, we use contraction to connect multiple CG matrices and tensor. This yields an equivariant mapping from tensor product of many spaces to a single space, which we call a path matrix. Then, a concatenation of path matrices becomes an equivariant mapping from tensor product of spaces to the direct sum of spaces. In the following we will introduce how we construct these path matrices in detail, and why they exhibit the properties we want.} We first obtain the following lemmas:

\lemma{
\label{lemma:contraction}
Given an $(l_x\otimes l_y,l_2)$-CG matrix $\hat{C}^{l_xl_yl_2}$ and an $(l_1, l_2,l_3)$-CG tensor $C^{l_1l_2l_3}$, if we let

\begin{equation}
\label{eq:contraction}
    T_{\changed{i_1 i_2 i_3}} = \sum\limits_{j_1} \hat{C}^{l_xl_yl_2}_{i_2j_1} C^{l_1l_2l_3}_{i_1j_1i_3}
\end{equation}
and we flatten \changed{the first two indices of} $T$, converting it into a $(2l_1+1)(2l_x+1)(2l_y+1)\times (2l_3+1)$ matrix. Then the following statements hold true:

\vspace{0.1cm}

1) $T$ is $O(3)$-$(l_1\otimes l_x \otimes l_y, l_3)$-equivariant.

\vspace{0.1cm}

2) The columns of $T$ are orthogonal. 

\vspace{0.1cm}

3) The columns of $T$ have the same $L_2$-norms.

}

\vspace{0.2cm}

\vspace{0.2cm}

\begin{figure*}[tb]

  \centering
  \includegraphics[width=0.85\linewidth]{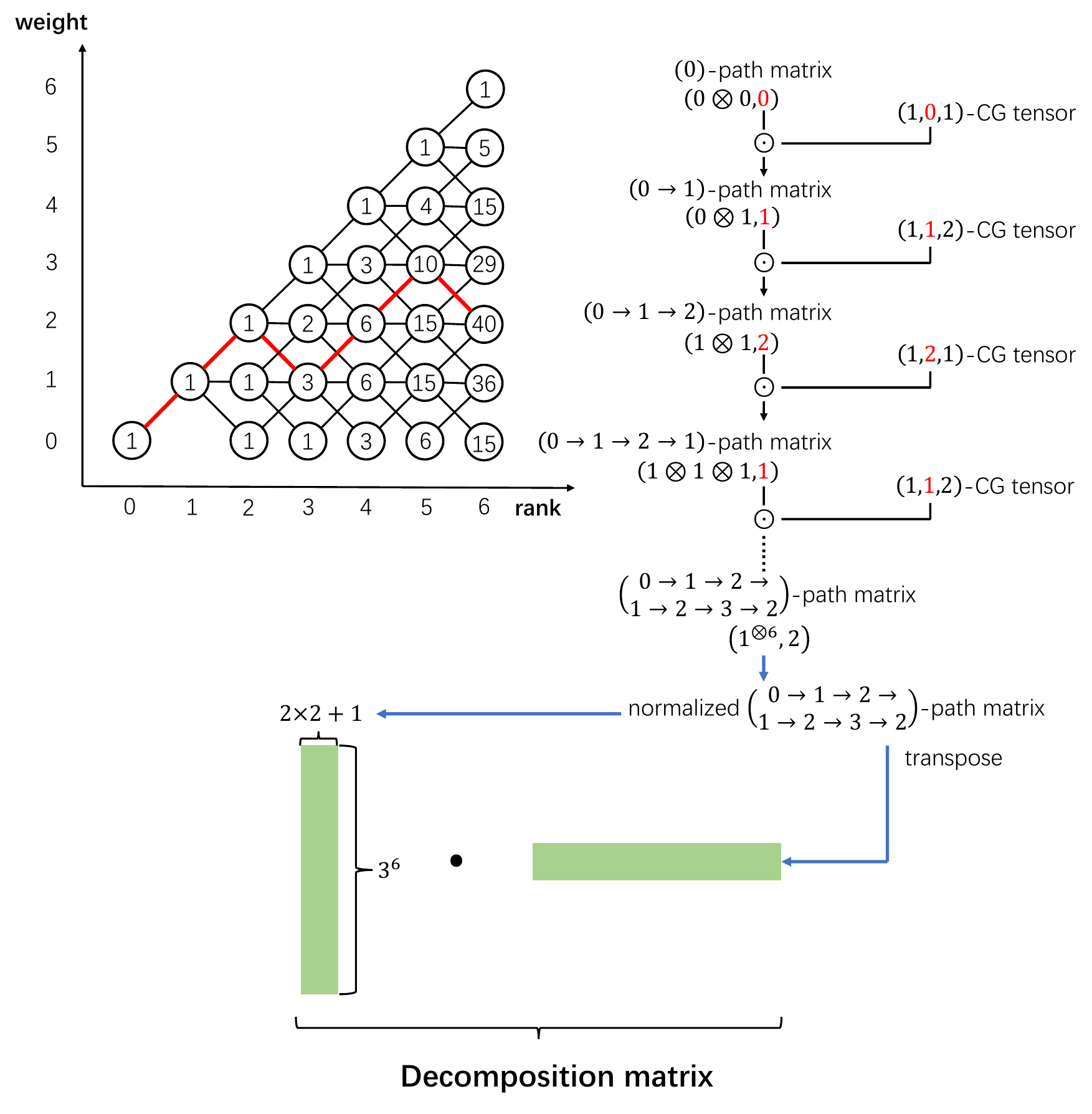}
  \vspace{-0.2cm}
  \caption{Construction of decomposition matrix. \textbf{Top:} The path matrices are formed via sequential contractions according to the paths in the parentage scheme. \textbf{Bottom:} The decomposition matrix is obtained as the product of the path matrix and its transpose.
  }
  \label{fig:ictdecomp}
\vspace{-0.5cm}
\end{figure*}

\noindent Here is the strategy for obtaining the path matrices, we first let a $(0\otimes 0,0)$-CG matrix to contract with a $(1,0,1)$-CG tensor to obtain a $(0\otimes 1,1)$-matrix. This matrix can further contract with $(1,1,2)$, $(1,1,1)$, and $(1,1,0)$-CG tensors to obtain the associated matrix. For example, we call the obtained $(0\otimes 1\otimes 1, 2)$ a $(0\to 1\to 2)$-path matrix $P^{(0\to 1\to 2)}$. The path matrices are constructed via the parentage scheme, which will decide what is the next CG tensor we are going to contract with, as shown in Figure \ref{fig:ictdecomp}. \changed{A rigorous description and the related properties are:}

\corollary{\label{coro:contract}\changed{Consider a path matrix $P^{(0\to l_1\to\cdots\to l_k)}$, with $l_1=1$, and $\abs{l_{i-1}-1}\leq l_i\leq l_{i-1}+1$, for $2\leq i\leq k$. This path matrix is iteratively constructed in the following way,
\begin{equation*}
P^{(0\to \dots\to l_u\to l_t)}_{i_1i_2i_3}= \sum_{j_1}P^{(0\to \dots\to l_u)}_{i_2j_1} C^{1, l_ul_t}_{i_1j_1i_3},
\end{equation*}
with the initial path matrix $P^{(0)}$ being a $(0\otimes 0,0)$-CG matrix, where $C^{1, l_2 l_3}$ is a CG tensor. After this contraction, we flatten the first two indices of $P^{(0\to \dots\to l_u\to l_t)}$. Then, the following properties hold,}

\vspace{0.1cm}

\changed{1) $P^{(0\to l_1\to\cdots\to l_k)}$ is $O(3)$-$((l=1)^{\otimes k}, l_k)$-equivariant.}

\vspace{0.1cm}

\changed{2) The columns of $P^{(0\to l_1\to\cdots\to l_k)}$ are orthogonal.}

\vspace{0.1cm}

\changed{3) The columns of $P^{(0\to l_1\to\cdots\to l_k)}$ have the same $L_2$-norms.}

}

\begin{proof}
    The path matrices have all the properties we used in Lemma \ref{lemma:contraction}. We prove this by a simple deduction. From the above lemma we know that $(0\to 1)$-path matrix holds all the properties. This gives us the initial condition. Assume that a matrix holds the properties, Lemma \ref{lemma:contraction} tells us the next obtained path matrix also shares them. 
\end{proof}

\vspace{0.2cm}

Next, we prove that

\lemma{\label{lemma:orthodifferent} The columns of different path matrices of the same path length are orthogonal. \changed{Specifically, given two path matrices $P^{(p)}$ and $P^{(p^{\prime})}$ with different paths $p$ and $p^{\prime}$, it holds that
\begin{equation*}
    \sum_{j}P^{(p)}_{ji_1} P^{(p^{\prime})}_{ji_2} = 0.
\end{equation*}}
}
\vspace{0.2cm}

\vspace{0.2cm}

Meanwhile, Schur's lemma is also required for the proof of the ICT decomposition.

\lemma[Schur's lemma]{\label{lemma:schur}
    Given finite-dimensional irreducible representations $\rho_{\mathcal{V}}:G\to \mathcal{V}$ and $\rho_{\mathcal{V}^\prime}:G\to \mathcal{V}^\prime$ of group $G$, and a matrix $M:\mathcal{V}\to\mathcal{V}^\prime$, suppose we have that

\begin{equation}
    \rho_{\mathcal{V}^\prime}(g)M = M\rho_{\mathcal{V}}(g)
\end{equation}
holds for any $g\in G$, then 

\vspace{0.2cm}

1) $M=0$ if $\mathcal{V}$ and $\mathcal{V}^\prime$ are not equivalent, i.e., there is no such matrix $Q$ that

\begin{equation}
    \rho_{\mathcal{V}}(g) = Q\rho_{\mathcal{V}^\prime}(g)Q^{-1}
\end{equation}
holds for any $g\in G$.

\vspace{0.2cm}

2) $M=cI$, where $c$ is a constant, if $\mathcal{V}=\mathcal{V}^\prime$.

}

\vspace{0.2cm}

Till now, it is adequate to show our solution to the problem:

\begin{algorithm}[t]
    \caption{Decomposition Matrices Generation}
    \label{algo:decomp}
    \begin{algorithmic}
        \State \changed{\textbf{Input:} Target rank $n$;}

        \State \changed{\textbf{Step 1.} Initialize $(0)$-path matrices $P^{(p_0)}=C^{000}$, and flatten the first two indices. Next, maintain the path $p_0=(0)$, the current column number of the parentage scheme $i=0$, and the current weight $l_0=0$.}

        \State \changed{\textbf{Step 2.} For those available weights in the $(i+1)$-th column, i.e., weights $\abs{l_i-1}\leq l_{i+1}\leq l_i+1$, we update the path by adding $l_{i+1}$ to the end of path $p_{i}$ to obtain $p_{i+1}$, and calculate the contraction,
        \begin{equation*}
            P^{(p_{i+1})} = \mathrm{flatten}_2(P^{(p_i)}\odot C^{1l_il_{i+1}}),
        \end{equation*}
        where $\mathrm{flatten}_2$ denotes flattening the first two indices. Then, we update the column number $i \gets i+1$.}

        \changed{\State \textbf{Step 3.} Repeat \textbf{Step 2} until $i = n$. For each obtained path matrices $P^{(p)}$, we normalize the columns by
        \begin{equation*}
            \hat{P}^{(p)}_{*,c}= \frac{P^{(p)}_{*,c}}{ \norm{P^{(p)}_{*,c}}_2}.
        \end{equation*}
        Then, we multiply each path matrix with its transpose,
        \begin{equation*}
            H^{(n;l;q)} = \hat{P}^{(p)}\cdot (\hat{P}^{(p)})^{\top}.
        \end{equation*}}
        
        \State \changed{\textbf{Output:} The list of all $H^{(n;l;q)}$.}
    \end{algorithmic}
\end{algorithm}

\theorem{\label{theorem:ict} \changed{As in Algorithm \ref{algo:decomp},} given the set of all the paths $\{p_i\}$ of the same length $n+1$ with $\#\{p_i\}=K$ and path matrices $\{P^{(p_i)}\}$, the set of $\left\{\left(\norm{P^{(p_i)}}_{col}\right)^{-2}P^{(p_i)}\left(P^{(p_i)}\right)^\top\right\}$ is a rank $n$ orthogonal ICT decomposition. }

\vspace{0.2cm}

\begin{proof}
     From Corollary \ref{coro:contract}, we know that each $P^{(p_i)}$ has full rank (but not a square matrix), and has the same $L_2$ norm of each column, $\norm{P^{(p_i)}}_{col}$. Subsequently, normalized path matrices $\hat{P}^{(p_i)}=\left(\norm{P^{(p_i)}}_{col}\right)^{-1} P^{(p_i)}$ are with a unit $L_2$ norm, as $\norm{\left(\norm{P^{(p_i)}}_{col}\right)^{-1} P^{(p_i)}}_{col}=1$. This does not affect the equivariance of the path matrices, as the $L_2$-norm is invariant. We concatenate all the normalized path matrices $\hat{P}^{(p_i)}$ along the last dimension, 
\begin{align}
    D^{(l=1)^{\otimes n}\to {\rm irr}((l=1)^{\otimes n})}=
    \begin{bmatrix}
        \left(\norm{P^{(p_1)}}_{col}\right)^{-1} P^{(p_1)} &\dots& \left(\norm{P^{(p_K)}}_{col}\right)^{-1} P^{(p_K)}
    \end{bmatrix}
\end{align}
where irr outputs all the possible spherical irreducible representation spaces. The above equation forms a $3^{n}\times 3^{n}$ matrix. In this proof, we will use $D$ to represent $D^{(l=1)^{\otimes n}\to {\rm irr}((l=1)^{\otimes n})}$ to save space. It is an orthogonal matrix: $D^\top D=DD^\top=I$, as its columns are all orthogonal and normalized to unit $L_2$ norm as we discussed above. From Corollary \ref{coro:contract}, we know that $D^\top$ is $O(3)$-$\left((l=1)^{\otimes n},{\oplus_{i}^ K}\left(l^{(i)}\right)\right)$-equivariant, where $l^{(i)}\in {\rm irr}((l=1)^{\otimes n})$. If we left multiply it with a selection matrix (a diagonal matrix with ones and zeros on the diagonal) $I^{(j)}$ to select out the $j$th irreducible spherical representation, also defined as the last weight of the path $p_j$, the resulting $I^{(j)}D^\top$ is $O(3)$-$\left((l=1)^{\otimes n},l^{(j)}\right)$-equivariant. Similarly, the inverse $D$ is $O(3)$-$\left({\oplus_{i}^ K}\left(l^{(i)}\right),(l=1)^{\otimes n}\right)$-equivariant. With the right multiplied $I^{(j)}$, it becomes $O(3)$-$\left(l^{(j)},(l=1)^{\otimes n}\right)$-equivariant. Together, $D I^{(j)}D^\top$ is $O(3)$-$\left((l=1)^{\otimes n},(l=1)^{\otimes n}\right)$-equivariant. From Lemma \ref{lemma:firstdegree}, this relation is equivalent to $O(3)$-$\left((\mathbbm{R}^3)^{\otimes n},(\mathbbm{R}^3)^{\otimes n}\right)$-equivariance in the $n$-rank Cartesian tensor product space. Meanwhile, 
\begin{equation}
D I^{(j)}D^\top=\hat{P}^{(p_j)}(\hat{P}^{(p_j)})^\top, 
\end{equation}
which follows from the fact that $D$ is an orthogonal matrix, and the concatenation of $\hat{P}$ forms $D$. It therefore meets condition 1) of Definition \ref{def:ICTdecomp}. 

Next we prove 2), it holds true because 

\begin{align}
    \sum_{j} \hat{P}^{(p_j)}\left(\hat{P}^{(p_j)}\right)^\top = \sum_{j} D I^{(j)}D^\top= D\left(\sum_{j} I^{(j)}\right)D^\top  =I.
\end{align}

Then, for 3), we prove that $\hat{P}^{(p_j)}\left(\hat{P}^{(p_j)}\right)^\top$ has matrix rank $2l^\prime+1$, where $l^\prime$ is the last weight of the path $p_j$. It is easy to prove, as

\begin{equation}
    {\rm Rank}\left(\hat{P}^{(p_j)}\left(\hat{P}^{(p_j)}\right)^\top\right) = {\rm Rank}\left(\hat{P}^{(p_j)}\right) = {\rm Rank}\left(\left(\hat{P}^{(p_j)}\right)^\top\right) = 2l^\prime+1.
\end{equation}

To prove 4), we first note that the tensor product space is a direct sum of multiple irreducible representation spaces, as

\begin{equation}
    \left(\mathbbm{R}^3\right)^{\otimes n} = \oplus^K_i \hat{\mathcal{V}}^{\left(\left(\mathbbm{R}^3\right)^{\otimes n};i\right)},
\end{equation}
where $\hat{\mathcal{V}}^{(\mathcal{V};i)}$ denotes the $i$th irreducible representation, residing in the $\mathcal{V}$ space. Meanwhile, note that $\left(\hat{P}^{(p_j)}\right)^\top$ maps from $\left(\mathbbm{R}^3\right)^{\otimes n}$ to $ l^{(j)}$. According to Lemma \ref{lemma:schur} (Schur's lemma), $\left(\hat{P}^{(p_j)}\right)^\top$ is also a non-zero mapping from $\hat{\mathcal{V}}^{\left(\left(\mathbbm{R}^3\right)^{\otimes n};i\right)}$ to $l^{(j)}$, for some $i$. Conversely, the same reasoning tells us that $\hat{P}^{(p_j)}$ is a mapping from $l^{(j)}$ to $\hat{\mathcal{V}}^{\left(\left(\mathbbm{R}^3\right)^{\otimes n};i\right)}$. Put them together, $\hat{P}^{(p_j)}\left(\hat{P}^{(p_j)}\right)^\top$ is an equivariant mapping from $\hat{\mathcal{V}}^{\left(\left(\mathbbm{R}^3\right)^{\otimes n};i\right)}$ to itself. Since $\hat{\mathcal{V}}^{\left(\left(\mathbbm{R}^3\right)^{\otimes n};i\right)}$ is an irreducible representation space, there is no non-trivial $O(3)$-invariant subspace of the output space, as the input space does. This concludes the proof for 4).

Finally, we prove the orthogonality, given $\hat{P}^{(p_k)} (\hat{P}^{(p_k)})^\top$ and $\hat{P}^{(p_{k^\prime})} (\hat{P}^{(p_{k^\prime})})^\top$ with different $k\neq k^\prime$, consider dot product of an arbitrary pair of their columns $c\neq c^\prime$, we have
\begin{align}
&\left(\hat{P}^{(p_k)}\left(\hat{P}^{(p_k)}\right)^\top\right)_{*,c}\odot \left(\hat{P}^{(p_{k^\prime})}\left(\hat{P}^{(p_{k^\prime})}\right)^\top\right)_{*,c^\prime} \nonumber\\
     =&\sum_{j_1}\left(\sum_{j_2}\hat{P}^{(p_k)}_{j_1 j_2} \hat{P}^{(p_k)}_{c j_2}\right)\left(\sum_{j^\prime_2}\hat{P}^{(p_{k^\prime})}_{j_1 j^\prime_2} \hat{P}^{(p_{k^\prime})}_{c^\prime j^\prime_2}\right)\nonumber\\
     =&\sum_{j_2,j^\prime_2}\hat{P}^{(p_k)}_{c j_2}\hat{P}^{(p_{k^\prime})}_{c^\prime j^\prime_2}\underbrace{\left(\sum_{j_1}\hat{P}^{(p_k)}_{j_1 j_2} \hat{P}^{(p_{k^\prime})}_{j_1 j^\prime_2}\right)}_{0} \nonumber\\
     = &0.
\end{align}
The last equality follows from Lemma \ref{lemma:orthodifferent}. From Lemma \ref{lemma:ortholemma}, it satisfies the orthogonality requirement of Definition \ref{def:ICTdecomp}, which is stronger than condition 5), so 5) is implicitly guaranteed. 
\end{proof}

The length $n$ of the path in the above theorem includes $0$ as the first waypoint of the path. \changed{Here, we discuss why our work obtains orthogonality, while the predecessors' works (i.e., \cite{coope1965irreducible,andrews1982irreducible,bonvicini2024irreducible}) do not. Given a rank-$n$ tensor $T$ to decompose, they first convert it to the weight-$l$ natural tensor space by the mapping tensor $\tilde{G}$. Then, another mapping tensor $G$ pulls the tensor back to the original rank-$n$ tensor space. Since $\tilde{G}$ filters out irreducible representations other than the one of weight $l$ and some index $q$, the final space is of rank $n$, weight $l$, and index $q$. $\tilde{G}$ and $G$ are calculated from natural projection tensors and FICTs, which are, in fact, multiple $\delta$s and Levi-Civita $\epsilon$ (can be seen as rank-$2$ and rank-$3$ tensors, respectively). There are different permutation strategies of these $\delta$s and $\epsilon$'s. We only need to pick up those linearly independent components, which does not guarantee that they are orthogonal. We can make $G$ and $\tilde{G}$ into matrices, and the ICT decomposition matrices are then $G\tilde{G}$. Since the above orthogonality does not hold, different $G$ or $\tilde{G}$ are not orthogonal either. Now, look back to our work, the ICT decomposition matrices are $\hat{P}(\hat{P})^{\top}$, where $(\hat{P})^{\top}$ converts the tensor to the spherical irreducible space, and $\hat{P}$ pulls back the tensor to the original space. The orthogonality of different ICT decomposition matrices comes from the column orthogonality of $\hat{P}$, which further comes from the orthogonality of CG tensor. This is the key reason why our work guarantees orthogonality while the predecessors works do not. Next, we show the complexity of our algorithm for obtaining ICT decomposition matrices.}

\begin{theorem}
\changed{Obtaining the ICT decomposition matrices is of exponential time and space complexity, in terms of rank $n$.}
\end{theorem}

\begin{proof}
\changed{For time complexity, note that Parentage Scheme is faster than fully traversing a ternary tree, which consists of $3^n$ operations. Next, we consider the algebraic complexity. Suppose that we are in the $i$-th column of Parentage Scheme. For each weight $l_i$, suppose that we maintain a path $p_i$, and there are at most three valid contractions,  $P^{(p_i)}$ contracts with $C^{(1l_i(l_i-1))}$, $C^{(1l_i(l_i))}$, and $C^{(1l_i(l_i+1))}$, respectively. Since $P^{(p_i)}$ has shape $3^{i}\times (2l_i+1)$, the number of operations of the above contractions are $3^{i}\times(2l_i+1)\times 3\times  (2l_i-1) $, $3^{i}\times(2l_i+1)\times 3\times  (2l_i+1) $, and $3^{i}\times(2l_i+1)\times 3\times  (2l_i+3) $, respectively. Note that $2l_i +3\leq 2i+3$, thus the total of the above three contractions occupy no more than $(2i+3)^23^{i+2}$ operations. The contractions will happen for at most $i$ times, so this layer consists of at most $(2i+3)^33^{i+2}$ operations. There are $n$ layers, and for each layer, the number of operations is no more than $(2n+3)^33^{n+2}$. Thus, all the $n$ layers contribute to at most $(2n+3)^43^{n+2}$ operations. Lastly, we consider the multiplication between path matrices. There are at most $3^{2n}$ pairs of path matrices, each of which holds at most $(2n+1)3^{2n}$ operations. Thus, there are at most $(2n+1)3^{4n}$ operations. Together with the previous path matrices generation, there are in total at most $(2n+1)^4(3^{n+2}+3^{4n})$ operations, which holds $\mathcal{O}(n^43^{4n})$ complexity. 
For space complexity,
    the largest matrices we maintain are the set of $\hat{P}^{(p)}$ for possible all paths $p$, of which the total size is $3^{n-1}\times 3^{n-1}$.
    }
\end{proof}

It is also noteworthy that the complexity cannot be lower than exponential in terms of rank $n$, as the ICT decomposition matrices themselves have an inherently exponential structure.

\begin{table}[]
\centering
\setlength{\tabcolsep}{5pt}
\renewcommand{\arraystretch}{1.5}
\centering
\begin{tabular}{lccccc}
\hline
Rank & 5               & 6   & 7   & 8     & 9     \\ \hline
Time & \textless{}0.1s & 1s  & 3s  & 11s & 4m32s \\ \hline
\end{tabular}
\caption{Benchmarking decomposition matrices generation speed.}
\label{tab:speedbenchmark}
\end{table}

\subsection{An Example of Constructing Path Matrices}
\label{sec:ictdecompexample}

 Here, we describe the process of constructing decomposition matrices using the parentage scheme. A path in the scheme uniquely corresponds to a path matrix. We start by creating a temporary path matrix, denoted as $(0\otimes0,0)$, which will be updated in each round. This matrix is reshaped from a $(0,0,0)$ CG tensor (merging the first two dimensions). The starting point is $(0,0)$ in the parentage scheme, where the first coordinate represents the rank, and the second represents the weight.

Next, we move from rank $0$ to rank $1$, i.e., from the first column to the second column. Observe that the only line connecting $(0,0)$ and $(1,*)$ is the one between $(0,0)$ and $(1,1)$. Therefore, the CG tensor to be contracted is $(1,0,1)$. The $1$ at the third place is determined by the second $1$ in $(1,1)$ in the parentage scheme. The $0$ at the second place is determined by the fact that our currently maintained temporary path matrix has $0$ in the last position of $(0\otimes0,0)$. The first $1$ is fixed and remains unchanged regardless of which lines we are moving along (this behavior will change for the general parentage scheme, as introduced in Section \ref{sec:extension}).

The contraction is performed between the last dimension of the temporary path matrix and the second dimension of the CG tensor, as shown in equation (\ref{eq:contraction}). By merging the first two dimensions of the resulting tensor, we obtain a $(0\to1)$ path matrix, which can be represented as $(0\otimes0\otimes1,1)$ (we sometimes omit the $0\otimes0$ for notational simplicity).

At this point, we are positioned at $(1,1)$. The connected points in the next column are $(2,0)$, $(2,1)$, and $(2,2)$. Consequently, the CG tensors to be contracted are $(1,1,0)$, $(1,1,1)$, and $(1,1,2)$, respectively. These contractions yield initial path matrices $(0\to1\to2)$, $(0\to1\to1)$, and $(0\to1\to0)$ with shapes $(1\otimes1,2)$, $(1\otimes1,1)$, and $(1\otimes1,0)$, respectively. This process is repeated iteratively until we reach rank $n$. The process and an example are illustrated in Figure \ref{fig:ictdecomp}. \changed{A detailed example for generating rank-$2$ ICT decomposition matrices are shown in Appendix \ref{sec:examplerank2}.}

We present the process in Algorithm \ref{algo:decomp}. We also benchmark the speed of our algorithm for obtaining the set of all decomposition matrices for different ranks, as shown in Table \ref{tab:speedbenchmark}. We run the experiments on 28-cores Intel\textsuperscript{\textregistered} Xeon\textsuperscript{\textregistered} Gold 6330 CPU @ 2.00GHz.

\section{Basis of the Equivariant Design Space}
\label{sec:basis}

Section \ref{sec:problembasis} implies that one of our goals is to find the basis of the equivariant design space $\{ M\in \mathbbm{R}^{3^n}\times\mathbbm{R}^{3^n} \ |\ \rho(g)M=M\rho(g) \}$, where $M$ are some linear mappings. Recall Lemma \ref{lemma:schur} (Schur's lemma), for an endomorphism space from a single finite-dimensional irreducible representation space to itself, the basis is simply an identity matrix $I$ and it is a single dimension space. Note that a rank-$n$ Cartesian tensor is a direct sum of several such irreducible representation spaces. In such situation, we obtain the following lemma:

\lemma{ \label{lemma:basis}
    Given a representation $\mathcal{V}=\oplus^n_i \mathcal{V}^{(i)}$ that can be written as the direct sum of several finite-dimensional irreducible representations $\mathcal{V}^{(i)}$. If they are isomorphic to each other, i.e.,
    \begin{equation*}
    \mathcal{V}^{(0)}\cong \ldots \cong \mathcal{V}^{(n)}, 
    \end{equation*}
    then the number of dimensions of the equivariant design space is $n^2$, and the basis consists of one-dimensional linear projections with a scaling factor from $\mathcal{V}^{(i)}$ to $\mathcal{V}^{(j)}$ for any $i$ and $j$.}
    
\vspace{0.2cm}

\begin{proof}
    First note that 
    \begin{align}
    (\mathcal{V}^{(i)}\oplus \mathcal{V}^{(j)})\to \mathcal{V}^\prime\cong (\mathcal{V}^{(i)}\to \mathcal{V}^{\prime})+(\mathcal{V}^{(j)}\to \mathcal{V}^{\prime}),
    \end{align}
    where $+$ represents linear combinations in the same codomain, and conversely
\begin{align}
    \mathcal{V}^\prime\to(\mathcal{V}^{(i)}\oplus \mathcal{V}^{(j)}) \cong (S^{\prime}\to \mathcal{V}^{(i)})\oplus(\mathcal{V}^{\prime}\to \mathcal{V}^{(j)}).
\end{align}
Given a linear projection $f:\oplus^n_i \mathcal{V}^{(i)}\to \oplus^n_i \mathcal{V}^{(i)}$, we have
\begin{align}
    &\oplus^n_i \mathcal{V}^{(i)}\to \oplus^n_i \mathcal{V}^{(i)} \nonumber\\
    \cong &\oplus^n_j \left(\oplus_i^n \mathcal{V}^{(i)}\to \mathcal{V}^{(j)}\right) \nonumber\\
    \cong &\oplus^n_j \left(+_i^n \left(\mathcal{V}^{(i)}\to \mathcal{V}^{(j)}\right) \right).
\end{align}
Note that the mapping spaces $(\mathcal{V}^{(i_1)} \to \mathcal{V}^{(j)})$ and $(\mathcal{V}^{(i_2)} \to \mathcal{V}^{(j)})$ form a basis for the sum $(\mathcal{V}^{(i_1)} \to \mathcal{V}^{(j)}) + (\mathcal{V}^{(i_2)} \to \mathcal{V}^{(j)})$ when $i_1 \neq i_2$. According to the definition of a direct sum, the union of the basis elements of spaces $\mathcal{A}$ and $\mathcal{B}$ must form a basis of $\mathcal{A} \oplus \mathcal{B}$. By Lemma \ref{lemma:schur}, each such mapping space is one-dimensional. Therefore, the mapping from $\oplus_{i=1}^n \mathcal{V}^{(i)}$ to itself, denoted as $\oplus_{i=1}^n \mathcal{V}^{(i)} \to \oplus_{i=1}^n \mathcal{V}^{(i)}$, has a total of $n^2$ basis elements of the form $\mathcal{V}^{(i)} \to \mathcal{V}^{(j)}$.
\end{proof}
This lemma hints us that the keypoint of the basis construction is to find a set of matrices that map between different irreducible Cartesian tensors with the same weight. Similarly to the ICT decomposition, it is easy to find such mappings in spherical spaces rather than the Cartesian ones. Once we find such mappings, we can also use change-of-basis matrices to pull us back to the Cartesian space. Meanwhile, proper construction and the orthogonality of the path matrices could help us find an orthogonal basis.

Before formally proving the following construction is an orthogonal basis, we first introduce the idea of the construction. From the perspective of matrices, it is equivalent to say that the matrices which can \textit{scale an irreducible representation and add it to another} (including itself) form a basis for the equivariant design space. Here, we provide an example.

\begin{equation}
\label{eq:addingmatrix}
\underset{A^{(p_i\to p_j)}}{\begin{bmatrix}
  \ddots & \vdots & \vdots & \vdots & \vdots & \vdots & \vdots & \reflectbox{$\ddots$}  \\
  \cdots & 0 & 0 & 0 & k & 0 & 0 & \cdots \\
  \cdots & 0 & 0 & 0 & 0 & k & 0 & \cdots \\
  \cdots & 0 & 0 & 0 & 0 & 0 & k & \cdots \\
  \cdots & 0 & 0 & 0 & 0 & 0 & 0 & \cdots \\
  \cdots & 0 & 0 & 0 & 0 & 0 & 0 & \cdots \\
  \cdots & 0 & 0 & 0 & 0 & 0 & 0 & \cdots \\
  \reflectbox{$\ddots$} & \vdots & \vdots & \vdots & \vdots & \vdots & \vdots & \ddots 
\end{bmatrix}}
\begin{bmatrix}
  \vdots \\
  v^{(p_j)}_{-1} \\
  v^{(p_j)}_{0} \\
  v^{(p_j)}_{1} \\
  v^{(p_i)}_{-1} \\
  v^{(p_i)}_{0} \\
  v^{(p_i)}_{1} \\
  \vdots
\end{bmatrix} = 
\begin{bmatrix}
  \vdots \\
  kv^{(p_i)}_{-1} \\
  kv^{(p_i)}_{0} \\
  kv^{(p_i)}_{1} \\
  0 \\
  0 \\
  0 \\
  \vdots
\end{bmatrix},
\end{equation}
where $A^{(p_i\to p_j)}$ adds $kv^{(p_i)}$ to the position of $v^{(p_j)}$, where $v^{(p_i)}$ belongs to the spherical space mapped by the $p_i$ path matrix, and $k$ is a scalar. Note that $A^{(p_i\to p_j)}$ resides in the spherical space and is an ($r=1$,$s=1$)-tensor. To construct a matrix that performs the adding operation in the irreducible Cartesian space, we need to convert it back to Cartesian space via the change-of-basis matrices $D$ and $D^{-1}=D^\top$. Since it is an ($r=1$,$s=1$)-tensor, we need two change-of-basis matrices to multiply with it to convert both the vector component and the covector component, where $D$ This gives us
\begin{align}
\label{eq:addingmatrix2}
&\underset{D}{\begin{bmatrix}
  \vdots \\
  (\hat{P}^{(p_j)}_{*,0})^\top \\
  (\hat{P}^{(p_j)}_{*,1})^\top \\
  (\hat{P}^{(p_j)}_{*,2})^\top \\
  (\hat{P}^{(p_i)}_{*,0})^\top \\
  (\hat{P}^{(p_i)}_{*,1})^\top \\
  (\hat{P}^{(p_i)}_{*,2})^\top \\
  \vdots
\end{bmatrix}}^\top
\underset{A^{(p_i\to p_j)}}{\begin{bmatrix}
  \ddots & \vdots & \vdots & \vdots & \vdots & \vdots & \vdots & \reflectbox{$\ddots$}  \\
  \cdots & 0 & 0 & 0 & k & 0 & 0 & \cdots \\
  \cdots & 0 & 0 & 0 & 0 & k & 0 & \cdots \\
  \cdots & 0 & 0 & 0 & 0 & 0 & k & \cdots \\
  \cdots & 0 & 0 & 0 & 0 & 0 & 0 & \cdots \\
  \cdots & 0 & 0 & 0 & 0 & 0 & 0 & \cdots \\
  \cdots & 0 & 0 & 0 & 0 & 0 & 0 & \cdots \\
  \reflectbox{$\ddots$} & \vdots & \vdots & \vdots & \vdots & \vdots & \vdots & \ddots 
\end{bmatrix}}
\underset{D^\top}{\begin{bmatrix}
  \vdots \\
  (\hat{P}^{(p_j)}_{*,0})^\top \\
  (\hat{P}^{(p_j)}_{*,1})^\top \\
  (\hat{P}^{(p_j)}_{*,2})^\top \\
  (\hat{P}^{(p_i)}_{*,0})^\top \\
  (\hat{P}^{(p_i)}_{*,1})^\top \\
  (\hat{P}^{(p_i)}_{*,2})^\top \\
  \vdots
\end{bmatrix}} \nonumber\\
=&\begin{bmatrix}
  \vdots \\
  (\hat{P}^{(p_j)}_{*,0})^\top \\
  (\hat{P}^{(p_j)}_{*,1})^\top \\
  (\hat{P}^{(p_j)}_{*,2})^\top \\
  (\hat{P}^{(p_i)}_{*,0})^\top \\
  (\hat{P}^{(p_i)}_{*,1})^\top \\
  (\hat{P}^{(p_i)}_{*,2})^\top \\
  \vdots
\end{bmatrix}^\top
\begin{bmatrix}
  \vdots \\
  k(\hat{P}^{(p_i)}_{*,0})^\top \\
  k(\hat{P}^{(p_i)}_{*,1})^\top \\
  k(\hat{P}^{(p_i)}_{*,2})^\top \\
  0 \\
  0 \\
  0 \\
  \vdots
\end{bmatrix} \nonumber\\
= &k\sum\limits_{j}\hat{P}^{(p_j)}_{i_1 j} \left(\hat{P}^{(p_i)}\right)^\top_{i_2 j}.
\end{align}
Thus, we can reduce the multiplication of three square matrices to a series of products of sub-matrices. To formally prove that the above construction works, we also need the help from the Sylvester’s inequality \citep{matsaglia1974equalities}.

\lemma[Sylvester’s inequality]{
\label{lemma:sylvester}
Given $M\in \mathbbm{C}^{m\times n}$ and $M^\prime\in \mathbbm{C}^{n\times p}$, we have

\begin{equation*}
    {\rm Rank}(MM^\prime) +n \geq {\rm Rank}(M)+{\rm Rank}(M^\prime),
\end{equation*}
where ${\rm Rank}$ outputs the matrix rank.
}

\vspace{0.2cm}

We are ready to show
\theorem{
\label{prop:basis}
Given a rank $n$ Cartesian tensor product space, $\hat{P}^{(p_i)}\left(\hat{P}^{(p_j)}\right)^{\top}$ for all pair of paths $p_i$ and $p_j$ of which the last weights are the same, form an orthogonal basis of the equivariant design space with respect to the Cartesian space. Here, the orthogonality means that the Frobenius inner products of the different path matrices are always zero,

\begin{equation}
    \left\langle \hat{P}^{(p_i)}\left(\hat{P}^{(p_j)}\right)^{\top},\hat{P}^{(p_{i^\prime})}\left(\hat{P}^{(p_{j^\prime})}\right)^{\top}\right\rangle_F=0, 
\end{equation}
for any $i\neq i^\prime$ or $j\neq j^\prime$. Given rank $n$, the dimension of the equivariant space equals
\begin{equation} \label{eq:basisnumber} \sum\limits^{\changed{n}}_{\changed{k=0}}\bigg( \sum\limits^{\lfloor (n-k)/3 \rfloor}_{\changed{u=0}} (-1)^u\frac{n!(2n-3u-k-2)!}{u!(n-u)!(n-2)!(n-3u-k)!}\bigg)^2. 
\end{equation} 
}

\begin{proof}
    To utilize Lemma \ref{lemma:basis}, we need to prove that each of the $\hat{P}^{(p_i)}\left(\hat{P}^{(p_j)}\right)^{\top}$ is a linear mapping from a weight $l$ ICT space to some other (including itself) ICT space. Thus, these matrix (linear mapping) need to satisfy the requirements of Definition \ref{def:ICTdecomp}, except for 2), i.e., the sum of these matrices need not to be an identity matrix. Also, the orthogonality is defined based on the Frobenius inner product.
    We first prove 1), the equivariance. For arbitrary $\hat{P}^{(p_i)}\left(\hat{P}^{(p_j)}\right)^{\top}$, and group element $g$, we have
    \begin{align}
    &\hat{P}^{(p_i)}\left(\hat{P}^{(p_j)}\right)^{\top} \rho_{(R^{3})^{\otimes n}}(g)=\hat{P}^{(p_i)}D^{-1}A^{(p_j\to p_j)} \rho_{(R^{3})^{\otimes n}}(g) \nonumber\\
    =&\hat{P}^{(p_i)}\rho_{l}(g)\left(\hat{P}^{(p_j)}\right)^{\top} = \rho_{(R^{3})^{\otimes n}}(g)\hat{P}^{(p_i)}\left(\hat{P}^{(p_j)}\right)^{\top},
    \end{align}
where $A$ is a selection matrix defined in equation (\ref{eq:addingmatrix}). Then, for 3), we prove that $\hat{P}^{(p_i)}\left(\hat{P}^{(p_j)}\right)^\top$ has matrix rank $2l^\prime+1$, where $l^\prime$ is the last weight of the path $p_j$. We know that the shapes of these two matrices are both $3^n\times \left(2l^\prime+1\right) $, and the transpose has shape $\left(2l^\prime+1\right) \times 3^n$. From the rank inequality, it holds that

\begin{align}
\label{eq:upper}
    &{\rm Rank}\left(\hat{P}^{(p_i)}\left(\hat{P}^{(p_j)}\right)^\top\right)\leq {\rm min}\left({\rm Rank}\left(\hat{P}^{(p_i)}\right),{\rm Rank}\left(\left(\hat{P}^{(p_j)}\right)^\top\right)\right) \nonumber\\
    =& 2l^\prime +1.
\end{align}
On the other hand, from Lemma \ref{lemma:sylvester}, we have that

\begin{align}
\label{eq:lower}
    &{\rm Rank}\left(\hat{P}^{(p_i)}\left(\hat{P}^{(p_j)}\right)^\top\right)\geq {\rm Rank}\left(\hat{P}^{(p_i)}\right)+{\rm Rank}\left(\left(\hat{P}^{(p_j)}\right)^\top\right)-2l^\prime-1 \nonumber\\
    =& 2(2l^\prime+1)-2l^\prime -1 \nonumber\\
    =& 2l^\prime+1.
\end{align}
Equations (\ref{eq:upper}) and (\ref{eq:lower}) together give us

\begin{equation}
    {\rm Rank}\left(\hat{P}^{(p_i)}\left(\hat{P}^{(p_j)}\right)^\top\right) =2l^\prime +1,
\end{equation}
which satisfies condition 3) of the definition.

The proof of 4) is almost identical to that of Theorem \ref{theorem:ict}.

To demonstrate orthogonality, we need to reconsider the construction of basis elements in its equivalent form: the transposed change-of-basis matrix $D^\top$ multiplied by the selection matrix $A$ and then by the change-of-basis matrix $D$, as shown in equation (\ref{eq:addingmatrix2}). It holds that
    \begin{align}
    \label{eq:basisproof1}
    &\left\langle \hat{P}^{(p_i)}\left(\hat{P}^{(p_j)}\right)^{\top},\hat{P}^{(p_{i^\prime})}\left(\hat{P}^{(p_{j^\prime})}\right)^{\top}\right\rangle_F\nonumber\\
    =&\left\langle D A^{(p_j\to p_i)}D^\top,DA^{(p_{j^\prime}\to p_{i^\prime})}D^\top\right\rangle_F\nonumber\\
    =&\left\langle \underbrace{D^\top D}_I A^{(p_j\to p_i)}\underbrace{D^\top D}_I,A^{(p_{j^\prime}\to p_{i^\prime})}\right\rangle_F\nonumber\\
    =&\left\langle A^{(p_j\to p_i)},A^{(p_{j^\prime}\to p_{i^\prime})}\right\rangle_F \nonumber\\
    =& 0,
    \end{align}    
    \noindent for any $i\neq i^\prime$ or $j\neq j^\prime$. Here $D^\top D=I$ holds because of the fact that the change-of-basis matrix $D$ is an orthogonal matrix, as we stated in the proof of Theorem \ref{theorem:ict}. The dimension calculation is based on Proposition \ref{prop:multip}.
\end{proof}
\vspace{0.2cm}

It is important to note that we cannot define a basis where each basis element maps a vector to an orthogonal one, as was done for the ICT decomposition matrices in Theorem \ref{theorem:ict}. The proof is simple: By calculating the cardinality of the basis using equation (\ref{eq:basisnumber}), we find that the cardinality of basis of the rank $5$ tensor product space is $603$. However, the number of dimensions in this tensor product space is $3^5=243$, which is less than $603$. Thus, it is impossible to generate so many orthogonal vectors using this equivariant basis. \changed{Instead, we define the orthogonality
based on the Frobenius norm. Frobenius orthogonality is a natural generalization of $L_2$ orthogonality for vectors. There are two main advantages of maintaining such orthogonality. First, it is very easy to project an arbitrary matrix onto an equivariant linear mapping. Specifically, let $\{\hat{M}_{i}\}_{1\leq i\leq n}$ be the orthogonal basis, and given a matrix $M$, we can immediately obtain the projection in the equivariant space by
\begin{equation*}
    \hat{M} = \sum^{n}_{i=1} \langle \hat{M}_i,M \rangle_{F} \hat{M}_i,
\end{equation*}
which is not easy for spectral norm, for example, as it is not induced by inner product. Second, given the orthonormal basis, it can avoid the situation that the weights blow up. For instance, given a 3-d space $(x,y,z)$, suppose that our target space is the 2-d space $z=0$, we have a basis of two elements to span this space, $(1,0,0)$ and $(1,10^{-5},0)$, which requires very high weights to merely describe $(1,1,0)$. Instead, if our basis is $(1,0,0)$ and $(0,1,0)$, we can just let $(1,1,0)=1*(1,0,0)+1*(0,1,0)$. So the Frobenius orthogonality enhances the numerical stability.} Note that the ICT decomposition matrices are also part of the basis of the equivariant space.

\section{Extension to General Spaces}
\label{sec:extension}
So far, we have introduced the method to find bases for equivariant design spaces where the input and the output spaces are the same. Here, we focus on extending the method to equivariant spaces between different spaces.
\subsection{General ICT Decomposition}

It is appealing to build equivariant layers between different desired spaces. So far, we have only discussed Cartesian tensor product spaces. The spherical tensor product spaces enable us to utilize high-order many-body information while remaining flexible in terms of memory cost and inference speed. However, it is difficult to construct such a layer. For example, if the input space is $(l=4 \otimes l=5)$ and the output space is $(l=2 \otimes l=6)$, how can we parameterize the equivariant layer? More specifically, what is the basis of this layer? It is even not possible to find a spanning set by \cite{pearce2023brauer}, as it is for Cartesian tensor product spaces. Before answering this question, we first introduce the general problem of ICT decomposition and how to efficiently perform a general ICT decomposition on arbitrary input spaces.

\begin{definition}
\label{def:generaldecomp}
    Given a space $\mathcal{V}=(l_{11}\otimes\dots\otimes l_{1i_{1}})\oplus\dots\oplus(l_{k1}\otimes\dots\otimes l_{ki_{k}})$ with dimension $d$, a general irreducible decomposition is a series of matrices $\hat{H}^{(\mathcal{V};l;q)}\in \mathbbm{R}^{d\times d}$, such that:

    \vspace{0.2cm}
    
    \qquad 1) $\hat{H}^{(\mathcal{V};l;q)}$ is $O(3)$-$(\mathcal{V},\mathcal{V})$-equivariant.

    \vspace{0.2cm}
    
    \qquad 2) The summation of the decomposition equals to identity matrix, $\sum_{l,q} \hat{H}^{(\mathcal{V};l;q)} = I$.

    \vspace{0.2cm}

     \qquad 3) The rank (matrix rank) of the decomposition matrix $\hat{H}^{(\mathcal{V};l;q)}$ is $2l+1$.

\vspace{0.2cm}

    \qquad 4) $\{\hat{H}^{(\mathcal{V};l;q)}v\mid v\in \mathcal{V}\}$ has no non-trivial $O(3)$-invariant subspace.

\vspace{0.2cm}

     \qquad 5) $\{\hat{H}^{(\mathcal{V};l;q)}v\mid v\in\mathcal{V}\}\bigcap \{\hat{H}^{(\mathcal{V};l^\prime;q^\prime)}v\mid v\in\mathcal{V}\}=\{0\}$ holds for $l\neq l^\prime$ or $q\neq q^\prime$.

\vspace{0.2cm}

If $\hat{H}^{(\mathcal{V};l;q)}$ are orthogonal to each other, as Lemma \ref{lemma:ortholemma}, we further call it an orthogonal general irreducible decomposition.
\end{definition}

Intuitively, an orthogonal general ICT decomposition can be obtained using a strategy similar to that of the classic ICT decomposition. However, we cannot use the parentage scheme directly, as it only accepts a tensor product with an $(l=1)$ one step at a time, proceeding from left to right. Here, we introduce a general parentage scheme. Unlike the classic scheme, this is not a fixed procedure.

\changed{Given a space 
\begin{equation*}
\mathcal{V}=\underbrace{(l_{11}\otimes\dots\otimes l_{1i_{1}})}_{\mathcal{V}_1}\oplus\dots\oplus\underbrace{(l_{k1}\otimes\dots\otimes l_{ki_{k}})}_{\mathcal{V}_k},
\end{equation*}
we construct the general parentage scheme as follows. For each $\mathcal{V}_{1\leq t\leq j}$, we read the first weight $l_{t1}$, of this tensor product space, and place multiplicity $1$ in the position of column $1$ (sometimes we omit drawing explicit x-axis) and row $l_{t1}$. This is what we draw in the first column. After we finish drawing in column $j$, we draw connected lines between the $j$-th and $(j+1)$-th columns, and write multiplicities of each row of the $(j+1)$-th column. To do this, we traverse each row of the multiplicities $q$ in the $j$-th column. For a given multiplicity $q$ of weight (row index) $l_j$, we pick up the row with weight $l_{j+1}$ in the next column $j+1$ that satisfies $\abs{l_j-l_{t(j+1)}}\leq l_{j+1}\leq l_j+l_{t(j+1)}$, then we write multiplicity $q$ in ($j+1$, $l_{j+1}$), and draw a line connecting ($j$,$l_j$) and ($j+1$,$l_{j+1}$). After that, we write $l_{t(j+1)}$ on top of the line. If this position already has a multiplicity, we add $q$ to the existed multiplicity. We repeat this process until we traverse to the end of this tensor product space. Different schemes from different tensor product spaces are drawn independently. Clearly, this is a generalization of the classic parentage scheme. The latter can be seen as a general parentage scheme when the only tensor product space is $(l=1)^{\otimes n}$, and the previous steps exactly follow. The general parentage scheme can handle those situations that the space is complicated rather than just a tensor product of Cartesian spaces.}

For example, given a space $(3 \otimes 4 \otimes 2) \oplus (2 \otimes 3)$, we first separate the terms by the direct sum $\oplus$, resulting in $(3 \otimes 4 \otimes 2)$ and $(2 \otimes 3)$. Then, we use the general parentage scheme to decompose $(3 \otimes 4 \otimes 2)$ into irreducible spherical spaces and construct the path matrices. We begin by writing a multiplicity of one in the weight-$3$ row, in the leftmost column. Next, we take the tensor product with $(l=4)$, of which the decomposition results in multiplicities of one in $(l=7), \dots, (l=1)$. We thereby construct the path matrices $(3 \overset{4}{\to} 7), \dots, (3 \overset{4}{\to} 1)$. Note that we mark $4$ on top of the arrow, which we refer to as the bridge number. Then, $2$ is taken as the bridge number, and we accordingly generate new path matrices. Consequently, we apply the same procedure for $(2 \otimes 3)$. An important property is

\corollary[from Lemma \ref{lemma:contraction}]{\label{coro:contractgeneral} Given a tensor product space, the path matrices for \textbf{the same} general parentage scheme have orthogonal columns, and the $L_2$ norms of columns of each path matrix are the same.}

\vspace{0.2cm}

\begin{proof}
    The proof is mostly identical to that of Lemma \ref{lemma:contraction}. 
\end{proof}

Here the words \textbf{the same} suggest that we need to fix our bridge number to obtain orthogonal path matrices. As our previous example, $(3\overset{4}{\to}7\overset{2}{\to}6)$ and $(4\overset{3}{\to}7\overset{2}{\to}6)$ will generate two different path matrices. But they are actually identical up to an index permutation. The result will be the same if we accordingly permute the index of our input vector. We can multiply the path matrices with their inverse to obtain decomposition matrices, as the classic ICT decomposition does. Here we should note that the decomposition matrices of different sub-schemes separated by the direct sum cannot be generally concatenated. We can, in general, concatenate some zeros to the path matrices and their transpose with short rows and columns, respectively. Or we can just leave the decomposition matrices of different shapes as they are, and multiply the matrices with different segments of the input vector. Constructing such path matrices, and multiplying with its inverse, can contribute to a general irreducible decomposition. The proof that it satisfies Definition \ref{def:generaldecomp} is very similar to that of Theorem \ref{theorem:ict}.

\subsection{Equivariant Bases of Arbitrary Spaces}

\label{sec:equiarbitrary}

The general ICT decomposition gives us some hints. The constructed change-of-basis matrix $D^\top$ changes the basis from $(3\otimes 4\otimes 2)\oplus(2\otimes 3)$ to its decomposed irreducible spherical spaces, denoted as ${\rm irr}\left(\left(3\otimes 4\otimes 2)\oplus(2\otimes 3\right)\right)$, and its inverse $D$ changes back the basis. An orthogonal basis from spaces $\mathcal{V}_1$ to $\mathcal{V}_2$ can be obtained via the change-of-basis matrices constructed based on the general ICT decompositions of these two spaces, but we need to take care of more when considering $O(3)$. The above discussion is identical for $SO(3)$ and $O(3)$. Instead, here, the parity must get involved when we consider arbitrary spaces. First, one should note that a tensor product space, as irreducible spherical spaces, have parity $1$ or $-1$. We have kept in mind that Cartesian space is of parity $-1$, which means that if we do the reflection, then the space will be reflected accordingly. It is equivalent to an $(l=1,p=-1)$ spherical space, from Lemma \ref{lemma:firstdegree}. Meanwhile, we should remember that there is also $(l=1,p=1)$ spherical space, as a \textit{cannot-be-reflected} version of the Cartesian space. The parity of the tensor product space is only decided by the number of the componential spaces with odd parities. If the number is odd, then the tensor product space has odd parity, and \textit{vice versa}.
Meanwhile, 

\proposition{Given a tensor product space of rank $n$, the resulting spaces of the ICT decomposition are with the same parity as the tensor product space.}

\vspace{0.2cm}

\begin{proof}
    Suppose that we have a rank $n$ tensor $T$ in the tensor product space, and assume that the parities of the ICTs and the original rank $n$ tensor $T$ are not the same. If a tensor $T$ has an even parity, then reflection will not change $T$. Let $\hat{T}^{(p=-1)}_1,\dots,\hat{T}^{(p=-1)}_k$ be the ICTs with odd parity. According to Definition \ref{def:generaldecomp}, we should have 
    \begin{align}
        T &= \sum_i \hat{T}_i^{(p=-1)} + \sum_{i^\prime} \hat{T}_{i^\prime}^{(p=1)} \nonumber\\
        &=-\sum_i \hat{T}_i^{(p=-1)} + \sum_{i^\prime} \hat{T}_{i^\prime}^{(p=1)},
    \end{align}
which gives us
\begin{align}
        \sum_i \hat{T}^{(p=-1)}_i=0.
\end{align}
This contradicts with the definition that $\hat{T}^{(p=-1)}_i$ are linearly independent. The same arguments can be given for the case that $T$ has odd parity.
\end{proof}

Now we are ready to introduce the procedure for obtaining an orthogonal basis for different spaces. Given spaces $\mathcal{V}^{in}$ and $\mathcal{V}^{out}$, we want to find a basis for the equivariant design space from $\mathcal{V}^{in}$ to $\mathcal{V}^{out}$. We first find the irreducible spherical spaces of $\mathcal{V}^{in}$ and $\mathcal{V}^{out}$ via the general parentage scheme, denoted as ${\rm irr}(\mathcal{V}^{in})$ and ${\rm irr}(\mathcal{V}^{out})$, and write down the paths. Then, the paths are grouped according to the last weight in them, as well as the parities of the original tensor product spaces, denoted as $Q^{(l;p)}_{\mathcal{V}^{in}}$ and $Q^{(l;p)}_{\mathcal{V}^{out}}$. Let our equivariant space be ${\rm Hom}_{O(3)}(\mathcal{V}^{in},\mathcal{V}^{out})$, following from Lemma \ref{lemma:schur} (Schur's lemma), we have 

\begin{equation}
    {\rm dim}\ {\rm Hom}_{O(3)}(\mathcal{V}^{in},\mathcal{V}^{out}) = \sum_{l,p}\#Q^{(l;p)}_{\mathcal{V}^{in}}\#Q^{(l;p)}_{\mathcal{V}^{out}},
\end{equation}
where $\#\cdot$ represents the number of elements in the set. We drop those paths sets that the last weight and the same parities do not show in both spaces. Then, we can generate path matrices based on those remaining sets. Finally, we conclude that 

\proposition{ \label{prop:equivbasisarbitrary}
Given spaces $\mathcal{V}^{in}=\left((l^{1}_1\otimes\dots\otimes l^{1}_{k_{1}}),\hat{p}^{1}\right)\oplus\dots\oplus\left((l^{t}_1\otimes\dots\otimes l^{t}_{k_{t}}), \hat{p}^t\right)$, and $\mathcal{V}^{out}=\left((l^{\prime1}_1\otimes\dots\otimes l^{\prime1}_{k^\prime_{1}}),\hat{p}^{\prime1}\right)\oplus\dots\oplus\left((l^{\prime t}_1\otimes\dots\otimes l^{\prime t^\prime}_{k^\prime_{t^\prime}}),\hat{p}^{\prime t^\prime}\right)$, where $l$ and $\hat{p}$ are weights and parities, respectively.
\begin{align}
\left\{\hat{P}^{(p^{(out;l;\hat{p};q^\prime)} )}\left(\hat{P}^{(p^{(in;l;\hat{p};q)})}\right)^\top \quad\bigg|\quad    
 p^{(in;l;\hat{p};q)}\in Q^{(l;\hat{p})}_{\mathcal{V}^{in}},p^{(out;l;\hat{p};q^\prime)}\in Q^{(l;\hat{p})}_{\mathcal{V}^{out}},\hat{p}\in\{-1,1\}   \right\} \nonumber
 \end{align}
 forms an orthogonal basis of ${\rm Hom}_{O(3)}(\mathcal{V}^{in},\mathcal{V}^{out})$.
 } 

\vspace{0.2cm}

\begin{figure*}[htb]

  \centering
  \includegraphics[width=0.8\linewidth]{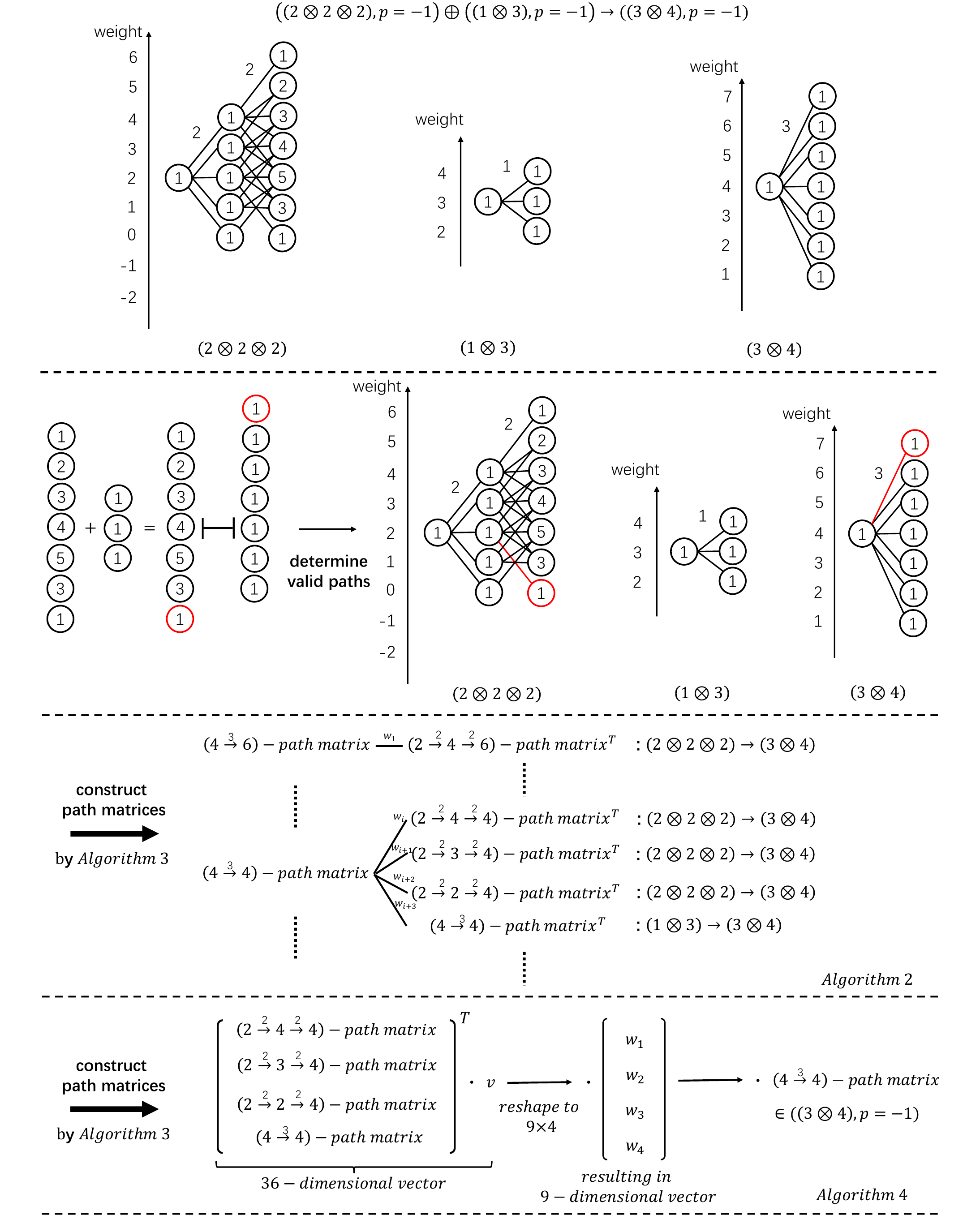}
  \caption{The general parentage scheme with the example of finding basis for $\left((2\otimes 2\otimes 2),p=-1\right)\oplus\left(\left(1\otimes 3\right),p=-1\right)\to\left((3\otimes 4),p=-1\right)$. Compared to the classic parentage scheme, we need to additionally mark the bridge number in the scheme.
  }
  \vspace{-0.8cm}
  \label{fig:ourscheme}

\end{figure*}

\begin{proof}
    The proof generally follows Theorem \ref{prop:basis}. Here, we focus on the orthogonality. Since $\hat{P}^{(p^{(out;l;\hat{p};q^\prime)})}\left(\hat{P}^{(p^{(in;l;\hat{p};q)})}\right)^\top$ are mappings between $\mathcal{V}^{in}$ and $\mathcal{V}^{out}$, they can be represented as a $\sum\limits^{t}_{j=1}\bigg(\Pi^{k_j}_{i=1}\left( 2l^{j}_{i}+1\right)\bigg)\times\sum\limits^{t^\prime}_{j=1}\left(\Pi^{k^\prime_j}_{i=1}\left( 2l^{j}_{i}+1\right)\right)$ matrix. From the mapping relations, we know that it is a block matrix as 

    \begin{align}
\begin{bmatrix}
M_{11} & \cdots & \cdots & M_{1t} \\
\vdots & \ddots & & \vdots \\
\vdots & & \ddots & \vdots \\
M_{t^\prime 1} & \cdots & \cdots & M_{t^\prime t}
\end{bmatrix}.
    \end{align}
The sub-matrices are of shape $\Pi^{k^\prime_{j^\prime}}_{i=1} \left(2l^{j^\prime}_{i}+1\right)\times\Pi^{k_j}_{i=1}\left( 2l^{j}_{i}+1 \right)$ for some $j$ and $j^\prime$. $\hat{P}^{(p^{(out;l;\hat{p};q^\prime)})}$  $\left(\hat{P}^{(p^{(in;l;\hat{p};q)})}\right)^\top$ will be in the same column for those $(l;\hat{p};q^\prime)$ belonging to the same direct sum component of $\mathcal{V}^{out}$, and the same row for those $(l;\hat{p};q)$ belonging to the same direct sum component of $\mathcal{V}^{in}$. When the non-zero values of $\hat{P}^{(p^{(out;l;\hat{p};q^\prime)})}\left(\hat{P}^{(p^{(in;l;\hat{p};q)})}\right)^\top$ appear in a block $M$, the rest of the sub-matrices will be zeros. Thus, given two elements in the basis, their corresponding Frobenius inner product must be zero if the non-zero values do not show in the same block $M$. If they appear in the same $M$, then the question is equivalent to Theorem \ref{prop:basis}, where we can prove that it is also zero.
\end{proof}


The algorithm for generating equivariant bases is presented in Algorithm \ref{algo:equivariantbasisgeneration}. Here, $CG\_decomp\_generator$ receives the current weight and the bridge number as input and outputs the possible weights according to the selection rules of CG coefficients. These rules depend on the group type. For instance, in the case of $O(3)$, $l_1$ and $l_2$ generate possible weights ranging from $|l_1 - l_2|$ to $l_1 + l_2$. 

\subsection{An Example for Obtaining a Basis of an Equivariant Space}

Here, we show an example for obtaining a basis of the equivariant space from 
\begin{equation*}
\left(\left(2\otimes2\otimes2\right),p=-1\right)\oplus\left(\left(1\otimes3\right),p=-1\right)
\end{equation*}
to 
\begin{equation*}
\left(\left(3\otimes4\right),p=-1\right). 
\end{equation*}

\begin{algorithm}[]

    \caption{Equivariant Basis Generation}
    \label{algo:equivariantbasisgeneration}
    \begin{algorithmic}
        \State \changed{\textbf{Input:} Input space 
        \begin{equation*}
\mathcal{V}_{in}=\underbrace{(l^{in}_{11}\otimes\dots\otimes l^{in}_{1i_{1}})}_{\mathcal{V}^{1}_{in}}\oplus\dots\oplus\underbrace{(l^{in}_{k1}\otimes\dots\otimes l^{in}_{ki_{k}})}_{\mathcal{V}^{k}_{in}},
\end{equation*}
and output space 
\begin{equation*}
\mathcal{V}_{out}=\underbrace{(l^{out}_{11}\otimes\dots\otimes l^{out}_{1i_{1}})}_{\mathcal{V}^{1}_{out}}\oplus\dots\oplus\underbrace{(l^{out}_{k1}\otimes\dots\otimes l^{out}_{ki_{k}})}_{\mathcal{V}^{k}_{out}}.
\end{equation*}}
\State \changed{\textbf{Step 1.} For each tensor product space $\mathcal{V}^{t}$ from $\mathcal{V}_{in}$ and $\mathcal{V}_{out}$, we draw the corresponding general parentage scheme as follows. Initialize a matrix $P=C^{000}$, and flatten the first two indices. Then, let it contract with CG tensor $C^{l_{t1}0l_{t1}}$, 
\begin{equation*}
    P^{(l_{t1})} = \mathrm{flatten}_2(P\odot C^{l_{t1}0l_{t1}}),
\end{equation*}
and from now on we maintain path $p_j$ with $p_1=l_{t1}$, where $j$ is the column index. We start from $j=1$.
}

\State \changed{\textbf{Step 2.} Suppose that we are at the $j$-th column, we pick up those available weights in the $(j+1)$-th column, i.e., weights $\abs{l_j-l_{t(j+1)}}\leq l_{j+1}\leq l_{t(j+1)}+1$, update the path by adding $l_{j+1}$ to the end of path $p_{j}$ to obtain $p_{j+1}$, and calculate the contraction,
        \begin{equation*}
            P^{(p_{j+1})} = \mathrm{flatten}_2(P^{(p_j)}\odot C^{l_{t(j+1)}l_jl_{j+1}}).
        \end{equation*}
Then, we update the column number $j \gets j+1$.}
        
\State \changed{\textbf{Step 3.} Repeat \textbf{Step 2} until $j$ reaches the number of components of the tensor product space. For each obtained path matrices $P^{(p)}$, we normalize its columns by
        \begin{equation*}
            \hat{P}^{(p)}_{*,c}= \frac{P^{(p)}_{*,c}}{ \norm{P^{(p)}_{*,c}}_2}.
        \end{equation*}}
\State \changed{\textbf{Step 4.} Repeat \textbf{Step 3} for each space, and divide the path matrices into two groups, from the input and the output space, respectively. }

\State \changed{\textbf{Step 5.} For each path matrix $\hat{P}^{(p)}_{in}$ from input space $\hat{P}_{in}$, multiply it with the transpose of one with the same termination of the path $p$, denoted by $\hat{P}^{(p^{\prime})}_{out}$ from the output space $\hat{P}_{out}$. Lastly, multiply a learnable parameter $w$, and we get $w\hat{P}^{(p^{\prime})}_{out}(\hat{P}^{(p)}_{in})^{\top}$.}
        
\State \changed{\textbf{Output:} The list of all $w\hat{P}^{(p^{\prime})}_{out}(\hat{P}^{(p)}_{in})^{\top}$.}
    \end{algorithmic}
\end{algorithm}

This process is illustrated in Figure \ref{fig:ourscheme}. First, we need to construct our general parentage schemes to identify all possible paths. For the tensor product $\left(2 \otimes 2 \otimes 2\right)$, the starting point of our scheme must be $2$. We then choose one of the remaining two $2$s as the bridge number, which gives us five possible paths, ranging from $(2 \overset{2}{\to} 0)$ to $(2 \overset{2}{\to} 4)$. Next, we select the last $2$ as the bridge number, leading to additional paths and resulting in a total of $19$ paths. We repeat this process to draw general schemes for $\left(1 \otimes 3\right)$ and $\left(3 \otimes 4\right)$, yielding $3$ and $7$ paths, respectively. Each scheme must have a parity, determined by the corresponding tensor product space; in this case, all the parities are equal to $-1$.
Next, we focus on the weights in the rightmost column of each scheme. We omit those weights that do not appear in both the input and the output spaces. In this example, the omitted weights are $0$ (which does not appear in the output space) and $7$ (which does not appear in the input space). We then construct the path matrices based on the remaining paths. The path matrices of the input and output spaces with matching last weights are grouped into pairs. The products of the path matrices of the output space and the transpose of the corresponding input space matrices form the basis elements of our equivariant maps. This procedure is described in Algorithm \ref{algo:equivariantbasisgeneration}.
Generating equivariant bases as matrices is elegant and enables efficient computations. These bases can also be efficiently obtained. However, when we tested the spanning set method proposed by \cite{pearce2023brauer}, assigning a scalar parameter to each basis and combining these bases to form a single matrix, we encountered memory issues in real-world neural network frameworks (e.g., PyTorch). This is because maintaining a large computation graph is challenging when the rank is very high. To address this issue in high-rank cases, it is preferable to store the path matrices with the same last weights for the input and output spaces separately.
For each available last weight, assuming we have $k$ paths for the input space and $k^\prime$ paths for the output space, we maintain a $k^\prime \times k$ matrix with learnable parameters. During training and inference, the inverse path matrices of the input space first convert the input into the spherical direct sum space. Then, the $k^\prime \times k$ matrices linearly combine the spherical vectors. Finally, the path matrices of the output space convert the vectors back to the output space. The entire process is also illustrated in Figure \ref{fig:ourscheme}. This approach follows from the previous proofs and aligns with the core idea of this paper: performing operations in spherical spaces is simpler and more efficient. This process is detailed in Algorithm \ref{algo:equivariantbasisgenerationlinearcombsphe}.

\begin{algorithm}[t]

    \caption{Equivariant Basis Generation---Linear Combinations in Spherical Spaces}
    \label{algo:equivariantbasisgenerationlinearcombsphe}
    \begin{algorithmic}
        \State \changed{\textbf{Input:} Input space 
        \begin{equation*}
        \mathcal{V}_{in}=\underbrace{(l^{in}_{11}\otimes\dots\otimes l^{in}_{1i_{1}})}_{\mathcal{V}^{1}_{in}}\oplus\dots\oplus\underbrace{(l^{in}_{k1}\otimes\dots\otimes l^{in}_{ki_{k}})}_{\mathcal{V}^{k}_{in}},
        \end{equation*}
        and output space 
        \begin{equation*}
        \mathcal{V}_{out}=\underbrace{(l^{out}_{11}\otimes\dots\otimes l^{out}_{1i_{1}})}_{\mathcal{V}^{1}_{out}}\oplus\dots\oplus\underbrace{(l^{out}_{k1}\otimes\dots\otimes l^{out}_{ki_{k}})}_{\mathcal{V}^{k}_{out}}.
        \end{equation*}}
        
        \State \changed{\textbf{Step 1.} Run the same \textbf{Steps 1-4} in Algorithm \ref{algo:equivariantbasisgeneration}, by which we obtain the path matrices $\hat{P}^{(p^{\prime})}_{out}$ and $\hat{P}^{(p)}_{in}$.}
        \State \changed{\textbf{Step 2.} Collect the last weight appearing in the path $p$ of each path matrix.}
        \State \changed{\textbf{Step 3.} For the input space, maintain those path matrices that have the corresponding matrices in the output space with the same last weight in the paths (one can also run this step for the output space without loss of generality).}
        \State \changed{\textbf{Step 4.} For each last weight, we calculate the number of path matrices whose path terminates at that weight, in both input and output spaces, denoted by $n$ and $m$, respectively.}
        \State \changed{\textbf{Step 5.} Create an $n\times m$ learnable weight matrix $W$ for each last weight.}
        \State \changed{\textbf{Output:} The list of $W$, $\hat{P}^{(p^{\prime})}_{out}$, and $\hat{P}^{(p)}_{in}$}
    \end{algorithmic}
\end{algorithm}

\subsection{Multi-Channel Features, Activation Functions, Normalizations, and Biases}

We have already introduced how to construct equivariant linear operations between some $\mathcal{V}^{in}$ and $\mathcal{V}^{out}$. In a real implementation, additional considerations are required to make it an equivariant layer in a modern neural network. First, if we take vectors in $\mathcal{V}^{in}$ and $\mathcal{V}^{out}$ as features, it is both standard and beneficial for a neural network to have multi-channel features. Fortunately, this is straightforward in our framework. We can simply let $\left(\mathcal{V}^{in}\right)^{\oplus K^{in}}$ and $\left(\mathcal{V}^{out}\right)^{\oplus K^{out}}$ serve as our input and output spaces, respectively. In this case, the conclusion in Section \ref{sec:equiarbitrary} can be seamlessly applied. Notably, Proposition \ref{prop:equivbasisarbitrary} already addresses the scenarios described here. Thus, for a space

\begin{align}
\underbrace{\left((l^{1}_1\otimes\dots\otimes l^{1}_{k_{1}}),p^{1}\right)}_{\mathcal{V}^{1}_{\oplus}}\oplus\dots\oplus\underbrace{\left((l^{t}_1\otimes\dots\otimes l^{t}_{k_{t}}),p^{t})\right)}_{\mathcal{V}^{t}_{\oplus}},
\end{align}
each direct sum component $\mathcal{V}^{i}_{\oplus}$ is a single-channel feature in a multi-channel view, which can be different depending on the positions. This is more general than the usual multi-channel features. 

On the other hand, activation functions introduce non-linearity to neural networks. As non-linear functions, their design space is very limited in order to preserve symmetry. However, it is completely safe if we can construct invariant linear mappings of the form $f^{inv}: v^{out} \in \mathcal{V}^{out} \mapsto c \in \mathbbm{R}$. Then, we can apply an activation function to $c$, obtaining $\phi(c)$, and multiply $\phi(c)$ with $v^{out}$ to produce the activated result $\phi(f^{inv}(v^{out}))v^{out}$. This process is equivariant, as
\begin{align}
    &\phi\left(f^{inv}\left(\rho_{\mathcal{V}^{out}}(g)\odot v^{out}\right)\right)\rho_{\mathcal{V}^{out}}(g)\odot v^{out} \nonumber\\
    = &\phi\left(\underbrace{\rho_{\mathbbm{R}}(g)}_{I}\odot f^{inv}\left( v^{out}\right)\right)\rho_{\mathcal{V}^{out}}(g)\odot v^{out} \nonumber\\
    =&\rho_{\mathcal{V}^{out}}(g)\odot\left(\phi\left( f^{inv}\left( v^{out}\right)\right)\odot v^{out} \right).
\end{align}
This process is similar to that of the spherical space-based methods \citep{batzner20223,musaelian2023learning}, but here the strategy is applicable to more general spaces.

For the normalizations, one can leverage the fact that the $L_2$-norm of the vectors are invariant to the rotation group, in general spaces. This is because the matrix representations of $O(n)$ always have an absolute determinant that is equal to one. Thus, for a vector $v\in \mathcal{V}$, we can simply calculate $\norm{v}_2$ to construct an equivariant normalization layer. The discussion on biases is mainly identical to \cite{pearce2023brauer}. We can construct equivariant basis from $\mathbbm{R}$ to arbitrary $\mathcal{V}^{out}$. In our work, we can find an orthogonal basis through Proposition \ref{prop:equivbasisarbitrary} instead of the spanning set. Here we should note that there are some situations that we cannot find such a mapping to $\mathcal{V}^{out}$. For example, $\mathbbm{R}$ with odd parity cannot map to $\mathcal{V}^{out}$ with even parity, and \textit{vice versa}. Also, $(l=0,p=1)$ cannot map to $(l=3,p=1)\otimes(l=2,p=1)$ either, as $(l=0,p=1)$ is not in their irreducible decomposition set ${\rm irr}\left( (l=3,p=1)\otimes(l=2,p=1) \right)=\{(l=5,p=1),(l=4,p=1),(l=3,p=1),(l=2,p=1),(l=1,p=1)\}$. Thus, it is convenient to always maintain $l=0$ irreducible representations for the ease of bias functions. Beyond the linear layer constraint, we can also build bias functions via multiplying $c\in \mathbbm{R}$ with each irreducible representations of unit $L_2$ norm, and adding the results with the irreducible representations.


\begin{table}[]
\centering
\setlength{\tabcolsep}{5pt}
\renewcommand{\arraystretch}{1.5}
\begin{tabular}{lcccccc}
\hline
Rank & 4               & 5   & 6   & 7   & 8 & 9     \\ \hline
EMLP \citep{finzi2021practical} & 2m52s           & 26m & OOM & OOM & OOM & OOM    \\
Algorithm \ref{algo:equivariantbasisgeneration} & \textless{}0.1s & 1s  & 3s  & 10s & OOM & OOM \\ 
Algorithm \ref{algo:equivariantbasisgenerationlinearcombsphe} & \textless{}0.1s & \textless{}0.1s  & \textless{}0.1s  & \textless{}0.1s & 0.8s & 3.4s \\ \hline
\end{tabular}
\caption{Comparison on the basis generation speed of ${\rm End}_{O(3)}\left(\left(\mathbbm{R}^{3}\right)^{\otimes n}\right)$. The out-of-memory error occurs when generating the bases using EMLP. In contrast, for Algorithm \ref{algo:equivariantbasisgeneration}, the error occurs when multiplying the learnable parameters.}
\label{tab:benchmarkbasis}
\end{table}

\section{Related Works}
\label{sec:relatedworks}

Research on equivariant neural networks is rooted in an abstract concept raised by \cite{cohen2016group}, where they extend the translation invariance of CNNs \citep{lecun1989backpropagation} to a wide range of groups. Meanwhile, graph neural networks \citep{gori2005new} are proven to be suitable tools for handling tasks involving point clouds. Their coupling, Equivariant Graph Neural Networks (EGNNs) has initiated an emerging area with very important applications to physics-informed tasks. Among the implementation, Message Passing Neural Networks (MPNNs) \citep{gilmer2017mpnn} are the most popular, where the constructed message is propagated to the neighboring points. Before that, \cite{schutt2017schnet} observe that group CNN frameworks require the points to lie on some grids and propose SchNet to flexibly handle atoms that randomly appear in space. \cite{gasteiger2020directional} further leverage directional information to enhance the model expressivity. At this point, the features utilized are just scalars and vectors, and subsequent works begin to use high-weight irreducible representations to effectively encode many-body interactions. \cite{thomas2018tfn} propose an abstract framework, the Tensor Field Network, to incorporate spherical irreducible representations into MPNNs. This abstract framework has been implemented with different designs to achieve SOTA performances on multiple popular benchmarks \citep{chmiela2017machine,ruddigkeit2012enumeration,fu2022forces}. \cite{gasteiger2021gemnet} demonstrate the universality of using such spherical irreducible representations. \cite{batzner20223} implement the Tensor Field Network with high-weight irreducible representations, showing strong performance. On the other hand, \cite{drautz2019atomic} proposes a local descriptor-based method, Atomic Cluster Expansion (ACE), to encode the atomic information via polynomial basis functions. \cite{batatia2022mace} further use high-order message-passing techniques to reduce the required number of layers utilizing spherical irreducible representations. \cite{musaelian2023learning} propose a strictly local EGNN, constructing high-weight spherical irreducible representations among edges. Recently, several studies leverage the fact that Cartesian tensor is more computationally efficient. \cite{simeon2024tensornet} first propose to use Cartesian tensor but only with rank-2 ICT decomposition, and they further apply it to charged molecules and spin states \citep{simeon2024inclusion}. \cite{cheng2024cartesian} uses high-rank Cartesian tensors but did not manipulate the irreducible components. \cite{zaverkin2024higher} successfully incorporate high-weight ICT to MPNNs but only address the highest weight of a given Cartesian tensor, lacking a full treatment of all ICTs for a given Cartesian tensor. 

Research on ICT decomposition has a long history and remains a highly active area in theoretical chemistry and chemistry physics communities, in addition to the recent strong interest from the machine learning community. The rank-$2$ ICT decomposition is well known. Since mid-20th century, \cite{coope1965irreducible} provide the explicit ICT decomposition for rank-3 Cartesian tensors. Seventeen years later, \cite{andrews1982irreducible} derives the rank-4 formula for this problem. \cite{dincckal2013orthonormal} further refines \cite{coope1965irreducible}'s results to make them orthogonal. Very recently, \cite{bonvicini2024irreducible} uses the RREF algorithm to handle the core part of finding linearly independent components in \cite{andrews1982irreducible}'s method, to successfully find rank-5 decomposition matrices. \cite{zou2001orthogonal} also propose a general method to perform the decomposition in a recursive style. However, this approach requires recursively decomposing each tensor, rather than having a decomposition tensor that can be directly applied to arbitrary tensors. Moreover, previous works rely on finding the linearly independent components of isotropic tensors, which scales in a factorial level, limiting the efficiency. \cite{snider2017irreducible} summarizes the development for ICT in his book. 

Meanwhile, researchers aim to understand what operations can be performed on a Cartesian tensor to maintain equivariance. This is one of the reasons ICTs are significant, as linear combinations of them are naturally equivariant. \cite{pearce2023brauer} leverages \cite{brauer1937algebras}'s theorem to obtain a spanning set for the equivariant design space. Based on this, \cite{pearce2023algorithm} further proposes an algorithm to generate the equivariant-preserving operation on the tensor product space. However, a direct construction of an orthogonal equivariant basis remains unknown. This work proposes a method to directly obtain orthogonal bases of equivariant design spaces.

\section{Discussion}
\label{sec:discussion}
In this section, we discuss the relationship between our work and the studies by \cite{pearce2023brauer} and \cite{finzi2021practical}. Additionally, we explore the application of our method to other groups.

\subsection{Relation to Other Works}
\label{sec:pearce-crump}

The remarkable work done by \cite{pearce2023brauer} has successfully found spanning sets for the equivariant design spaces for various groups. Here, we discuss the relation to this work. First, we recall the theorem for $O(n)$:

\begin{figure*}[tb]

  \centering
  \includegraphics[width=0.55\linewidth]{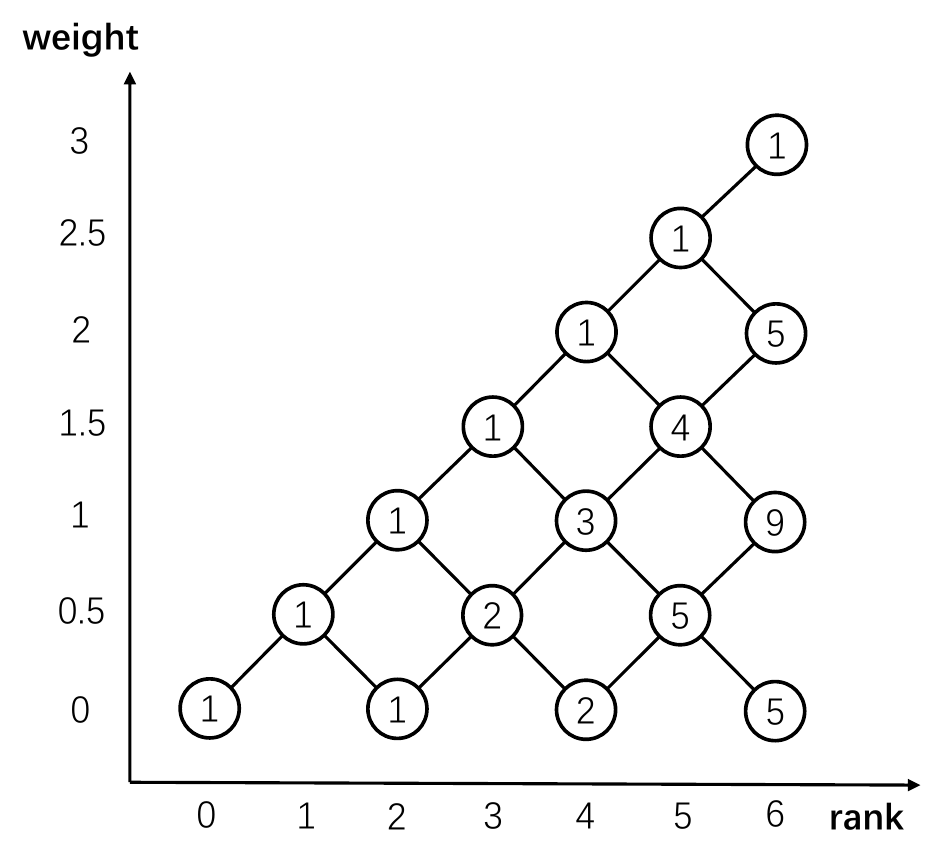}
  \vspace{-0.4cm}
  \caption{The parentage scheme of decomposition of $(l=0.5)^{\otimes n}$ for $SU(2)$ and $U(2)$, where $n$ is the rank. The numbers in the circle represent multiplicity.
  }
  
  \label{fig:su2scheme}

\end{figure*}

\theorem[\cite{pearce2023brauer}]{$U_{i_1,j_1}U_{i_2,j_2}\dots U_{i_\frac{l+k}{2}j_{\frac{l+k}{2}}}$} form a spanning set for ${\rm Hom}_{O(n)}\left(\left(\mathbbm{R}^n\right)^{\otimes l}, \left(\mathbbm{R}^n\right)^{\otimes k}\right)$, where $U$ are arbitrary matrices (also rank-2 tensors), $i_i$ and $j_i$ are different permutations of integer $1,\dots,\frac{l+k}{2}$. ${\rm Hom}_{O(n)}\left(\left(\mathbbm{R}^n\right)^{\otimes l}, \left(\mathbbm{R}^n\right)^{\otimes k}\right)=\varnothing$ when $l+k$ is odd, so the above $\frac{l+k}{2}$ is well defined.

\vspace{0.2cm}

Compared to his theorem, we identify orthogonal bases between arbitrary spaces where the input and the output spaces can be any spherical tensor product spaces, not limited to the Cartesian tensor product spaces. For example, \cite{pearce2023brauer}'s theorem cannot handle $((l=3\otimes l=4),p=-1)\to ((l=2\otimes l=2),p=-1)$, but our method can well address it, with an orthogonal basis instead of the spanning set. As he has noted, there is no equivariant operation between Cartesian spaces of odd and even ranks, respectively, while through our method we can find an orthogonal basis if their parities are the same. 

There is also a striking advance made by \cite{finzi2021practical}, who aim to find the basis of equivariant spaces. However, they adopt a numerical algorithm, EMLP, while we provide an analytical solution. As noted by \cite{pearce2023brauer}, \cite{finzi2021practical} cannot generate bases for continuous group tensor product spaces with higher ranks due to memory constraints and speed limitations. To further evaluate the performance of basis generation, we compare our method with that of \cite{finzi2021practical} using their open-sourced library. The results are presented in Table \ref{tab:benchmarkbasis}, where EMLP takes 26 minutes to compute the basis of ${\rm End}_{O(3)}\left(\left(\mathbbm{R}^{\otimes 5}\right)\right)$. In contrast, our method computes the same basis in less than one second and is capable of obtaining the orthogonal basis of ${\rm End}_{O(3)}\left(\left(\mathbbm{R}^{\otimes 8}\right)\right)$.


\subsection{Other Groups}
\label{sec:othergroups}
Although $O(3)$ is the most popular and useful group considered in equivariant neural networks, our conclusion can also be extended to $O(n)$, $SO(n)$, and $SU(n)$. The tensor product spaces can always be decomposed to smaller irreducible representation spaces, and we can always construct CG coefficients with orthogonality and normalization. The main challenge lies in obtaining the CG coefficients. Once they are obtained, the procedures outlined in this work can be repeated for those groups, and Proposition \ref{prop:equivbasisarbitrary} will give us the equivariant basis. Fortunately, we already have a robust method for computing CG coefficients for $SU(n)$, proposed by \cite{alex2011numerical}, who also provide an online service to freely compute the CG coefficients for $SU(n)$. One can generally follow the procedure of our work to extend the method to $SU(n)$. There are several relations between $SU(n)$ and certain $SO(n)$. For example, $SO(3)$ is double covered by $SU(2)$, which allows us to directly obtain CG coefficients for $SO(3)$ from those for $SU(2)$. The relationship stems from the fact that $SPIN(n)$ double covers $SO(n)$, and $SU(2)$ is isomorphic to $SPIN(3)$. Further, $SU(2)\otimes SU(2)$ is isomorphic to $SPIN(4)$, so it double covers $SO(4)$. Similarly, $SU(4)$ is isomorphic to $SPIN(6)$, and thus it double covers $SO(6)$. Consequently, the conclusion can also be extended to those $SO(n)$ that are covered by $SU(n^\prime)$ for high $n$. Further research into the Lie algebra of $SO(n)$ has the potential to directly discover the CG coefficients of $SO(n)$, rather than relying on the relationship with $SU(n)$. Additionally, $SO(n)$ can be easily converted to $O(n)$ by considering the parity. We additionally provide code for obtaining basis for equivariant spaces under $U(2)$ and $SU(2)$ in our code base. The parentage scheme is different from that of $O(3)$ and $SO(3)$, as shown in Figure \ref{fig:su2scheme}. 

\section{Conclusion}
\label{sec:conclusion}

In this work, we leverage the relationship between spherical spaces and their tensor product spaces, and construct path matrices to propose methods to: 1) find high-rank ICT decomposition matrices for $6\leq n\leq 9$, 2) find orthogonal bases for equivariant layers where the input and the output spaces are identical, and 3) generalize the conclusions to arbitrary pairs of spaces. The conclusions made in this work hold true for $O(n)$, $SO(n)$, and $SU(n)$. However, currently, we can only efficiently calculate $SU(n)$, $O(3)$, and $SO(3)$ due to the availability of known CG coefficients. This also highlights our future work, which will focus on achieving efficient calculations or closed-form formulae for $O(n)$ and $SO(n)$. Furthermore, the discovery of CG coefficients for more general groups can also help us find the basis of equivariant spaces for those groups. The contributions of this work can be directly applied to EGNN designs, which has already been shown to facilitate numerous applications in physics, chemistry, and broader biomedical fields. Examples include machine learning for force fields, drug design, blind docking, molecule generation, and many others. Additionally, there is a growing demand for equivariant methods in robotics \citep{wangequivariant}. Thus, our work has the potential to positively impact these related fields. For high-rank ICT decomposition, there are always requirements for its use in theoretical chemistry research, e.g., spectroscopies. Despite these advantages, it is essential to establish strict regulations for methods that could potentially harm the health of both people and animals.

\section{Reproducibility Statement}
\label{sec:repro}

The code to efficiently implement the theory is provided in \url{https://github.com/ShihaoShao-GH/ICT-decomposition-and-equivariant-bases}. The only external dependency packages are PyTorch and e3nn. The e3nn package can be easily installed using the command \texttt{pip install --upgrade e3nn}. The code can run on any computer with Python version $>$ 3.0.0 and PyTorch version $>$ 1.0. The code for generating the basis for $SU(2)$ does not require the e3nn package. However, the code for generating Clebsch-Gordan (CG) coefficients is borrowed from \url{http://qutip.org/docs/3.1.0/modules/qutip/utilities.html}. The CG coefficients for $SU(n)$ can be obtained from

\noindent https://homepages.physik.uni-muenchen.de/\textasciitilde vondelft/Papers/ClebschGordan/, as provided by \cite{alex2011numerical}. The code for generating these coefficients is detailed on pages 19–34 of their paper. The comparison with EMLP is based on their code base \url{https://github.com/mfinzi/equivariant-MLP}.

\section*{Acknowledgments}

Z. Lin was supported by the National Natural Science Foundation of China (No. 62276004) and the Beijing Natural
Science Foundation (No. L257007). Q. Cui was supported by grant from the National Natural Science Foundation of China (No. 62025102). S. Shao acknowledges the support from the Principal's Scholarship of Peking University.

\appendix

\changed{\section{Proofs}
\label{sec:proofs}
}

In this section, we provide proofs for auxiliary lemmas.

\setcounter{theorem}{11}

\lemma{ Given matrices $M:\mathcal{V}\to \mathcal{V}^\prime$ and $M^\prime:\mathcal{V}\to \mathcal{V}^\prime$ of the same shape, if $M^\top M^\prime=0$, then 
\begin{equation*}
    (M v)\odot(M^\prime v^\prime) = 0
\end{equation*}
for any $v,v^\prime\in \mathcal{V}$.
}

\vspace{0.2cm}

\begin{proof}
It holds true that
\begin{align}
    &(Mv)\odot(M^\prime v^\prime)  \nonumber\\ =&\sum_{i}\left(\sum_{j}M_{ij}v_{j}\right)\left(\sum_{j^\prime}M^\prime_{ij^\prime} v^\prime_{j^\prime}\right) \nonumber\\
    =&\sum_{j,j^\prime} v_{j}v^\prime_{j^\prime}\left(\sum_{i} M_{ij}M^\prime_{ij^\prime}\right). \label{eq:ortholemmaeq}
\end{align}
Assume we have $M^\top M^\prime=0$, then for any $j$ and $j^\prime$, we have

\begin{align}
    \sum_{i}M_{ij}M^\prime_{ij^\prime}=0.
\end{align}
Plugging it back to equation (\ref{eq:ortholemmaeq}), we obtain 

\begin{align}
    (Mv)\odot(M^\prime v^\prime)=\sum_{j,j^\prime} v_{j}v^\prime_{j^\prime}\left(\sum_{i} M_{ij}M^\prime_{ij^\prime}\right)=0.
\end{align}
This concludes the proof.

\end{proof}

\lemma{
Given an $(l_x\otimes l_y,l_2)$-CG matrix $\hat{C}^{l_xl_yl_2}$ and an $(l_1, l_2,l_3)$-CG tensor $C^{l_1l_2l_3}$, if we let

\begin{equation*}
    T_{\changed{i_1 i_2 i_3}} = \sum\limits_{j_1} \hat{C}^{l_xl_yl_2}_{i_2j_1} C^{l_1l_2l_3}_{i_1j_1i_3}
\end{equation*}
and we flatten $T$ into a $(2l_1+1)(2l_x+1)(2l_y+1)\times (2l_3+1)$ matrix, then the following statements hold true:

\vspace{0.1cm}

1) $T$ is $O(3)$-$(l_1\otimes l_x \otimes l_y, l_3)$-equivariant.

\vspace{0.1cm}

2) The columns of $T$ are orthogonal. 

\vspace{0.1cm}

3) The columns of $T$ have the same $L_2$-norms.

}

\vspace{0.2cm}

\begin{proof}
    We first prove 1), given an arbitrary group element $g\in O(3)$, we have that

    \begin{align}
        &\rho_{l_1\otimes l_x \otimes l_y}(g)\odot T \nonumber\\
        =&\rho_{l_1\otimes l_x \otimes l_y}(g)\odot \left(\sum_{j_1} \hat{C}^{l_xl_yl_2}_{i_2j_1} C^{l_1l_2l_3}_{i_1j_1i_3}\right) \nonumber\\
        = &\sum\limits_{i^\prime_1,i^\prime_2}R^{l_1}_{i_1i^\prime_1} R^{l_x\otimes l_y}_{i_2i^\prime_2}\left(\sum\limits_{j_1} \hat{C}^{l_xl_yl_2}_{i^\prime_2j_1} C^{l_1l_2l_3}_{i^\prime_1j_1i_3}\right) \nonumber\\
        = &\sum\limits_{j_1}\left(\sum\limits_{i^\prime_2}R^{l_x\otimes l_y}_{i_2i^\prime_2}\hat{C}^{l_xl_yl_2}_{i^\prime_2j_1}\right) \left(\sum\limits_{i^\prime_1}R^{l_1}_{i_1i^\prime_1}  C^{l_1l_2l_3}_{i^\prime_1j_1i_3}\right) \nonumber\\
        \overset{(a)}{=} &\sum\limits_{j_1}\left(\sum\limits_{j^\prime_1} \hat{C}^{l_xl_yl_2}_{i_2j^\prime_1} R^{l_2}_{j^\prime_1 j_1}  \right)\left(\sum\limits_{i^\prime_1} R^{l_1}_{i_1i^\prime_1}C^{l_1l_2l_3}_{i^\prime_1j_1i_3}\right) \nonumber\\
        = &\sum\limits_{j_1} \hat{C}^{l_xl_yl_2}_{i_2j_1}\left(\sum\limits_{i^\prime_1,j^\prime_1} R^{l_2}_{j^\prime_1 j_1}R^{l_1}_{i_1i^\prime_1}C^{l_1l_2l_3}_{i^\prime_1j^\prime_1i_3} \right) \nonumber\\
        \overset{(b)}{=} &\sum\limits_{j_1,i^\prime_3} \hat{C}^{l_xl_yl_2}_{i_2j_1} C^{l_1l_2l_3}_{i_1j_1i^\prime_3} R^{l_3}_{i^\prime_3i_3} \nonumber\\
         =& T\odot \rho_{l_3}(g),
    \end{align}
    so $T$ is $O(3)$-$(l_1\otimes l_x \otimes l_y, l_3)$-equivariant, where $R^{l}$ is the matrix representation in $l$ space, equalities (a) and (b) hold true because of Lemma \ref{lemma:equi}.

Next we prove 2), for an arbitrary pair of the indices of two columns $k\neq k^\prime$, we have
\begin{align}
    &\sum\limits_{j}T_{jk}T_{j k^\prime} \nonumber\\
    = &\sum\limits_{j_1,j_3}\left(\sum\limits_{j_2}\hat{C}^{l_xl_yl_2}_{j_1j_2}C^{l_1l_2l_3}_{j_3j_2k}\right)\left(\sum\limits_{j^\prime_2}\hat{C}^{l_xl_yl_2}_{j_1j^\prime_2}C^{l_1l_2l_3}_{j_3j^\prime_2k^\prime}\right)\nonumber\\
    = &\sum\limits_{j_2,j^\prime_2}\left(\sum\limits_{j_1}\hat{C}^{l_xl_yl_2}_{j_1j_2}\hat{C}^{l_xl_yl_2}_{j_1j^\prime_2}\right)\left(\sum_{j_3}C^{l_1l_2l_3}_{j_3j_2k}C^{l_1l_2l_3}_{j_3j^\prime_2k^\prime}\right)\nonumber\\
    \overset{(a)}{=} &\sum\limits_{j_1,j_2}\left(\hat{C}^{l_xl_yl_2}_{j_1j_2}\right)^2\left(\sum_{j_3}C^{l_1l_2l_3}_{j_3j_2k}C^{l_1l_2l_3}_{j_3j_2k^\prime}\right)\nonumber\\
    = &\norm{\hat{C}^{l_xl_yl_2}}_{col}^2\left(\sum_{j_2,j_3}C^{l_1l_2l_3}_{j_3j_2k}C^{l_1l_2l_3}_{j_3j_2k^\prime}\right)\nonumber\\
    \overset{(b)}{=} &0,
\end{align}
and thus the columns of $T$ are orthogonal. Equalities (a) and (b) follow from Lemma \ref{lemma:ortho} that $\sum\limits_{j_1}\hat{C}^{l_xl_yl_2}_{j_1j_2}\hat{C}^{l_xl_yl_2}_{j_1j^\prime_2}=0$ for $j_2\neq j^\prime_2$, and $\sum\limits_{j_2,j_3}C^{l_1l_2l_3}_{j_3j_2k}C^{l_1l_2l_3}_{j_3j_2k^\prime}=0$.

Finally, we prove 3), for an arbitrary column with index $k$, 
\begin{align}
    &\sqrt{\sum\limits_{j}T_{jk}T_{jk}}\nonumber\\
    = &\sqrt{\sum\limits_{j_1,j_3}\left(\sum\limits_{j_2}\hat{C}^{l_xl_yl_2}_{j_1j_2}C^{l_1l_2l_3}_{j_3j_2k}\right)\left(\sum\limits_{j^\prime_2}\hat{C}^{l_xl_yl_2}_{j_1j^\prime_2}C^{l_1l_2l_3}_{j_3j^\prime_2k}\right)}\nonumber\\
    = &\sqrt{\sum\limits_{j_2,j^\prime_2}\left(\sum\limits_{j_1}\hat{C}^{l_xl_yl_2}_{j_1j_2}\hat{C}^{l_xl_yl_2}_{j_1j^\prime_2}\right)\left(\sum\limits_{j_3}C^{l_1l_2l_3}_{j_3j_2k}C^{l_1l_2l_3}_{j_3j^\prime_2k}\right)}\nonumber\\
    = &\sqrt{\sum\limits_{j_2}\left(\sum\limits_{j_1}\left(\hat{C}^{l_xl_yl_2}_{j_1j_2}\right)^2\sum\limits_{j_3}\left(C^{l_1l_2l_3}_{j_3j_2k}\right)^2\right)} \nonumber\\
    = &\norm{\hat{C}^{l_xl_yl_2}}_{\rm col}\sqrt{\sum\limits_{j_2,j_3}\left(C^{l_1l_2l_3}_{j_3j_2k}\right)^2} \nonumber\\
    = &\norm{\hat{C}^{l_xl_yl_2}}_{\rm col} \norm{C^{l_1l_2l_3}}_{\rm col}.
\end{align}
From Lemma \ref{lemma:norm}, we know that it takes the same value for different $k$, which satisfies 3).

\end{proof}

\setcounter{theorem}{14}

\lemma{ The columns of different path matrices of the same path length are orthogonal. \changed{Specifically, given two path matrices $P^{(p)}$ and $P^{(p^{\prime})}$ with different paths $p$ and $p^{\prime}$, it holds that
\begin{equation*}
    \sum_{j}P^{(p)}_{ji_1} P^{(p^{\prime})}_{ji_2} = 0.
\end{equation*}}}

\vspace{0.2cm}

\begin{proof}
    We prove by induction. The first to third columns of the Parentage Scheme generate the $(0\to1\to2)$, $(0\to1\to1)$, and $(0\to1\to0)$ path matrices. We first demonstrate that the columns of these matrices are orthogonal. Since they originate from the contraction of $(0\to 1)$ path matrices and different CG tensors, we can establish a more general conclusion. Specifically, for any path matrix $P$ and two distinct CG tensors $C$ and $C^\prime$, given column indices $k$ and $k^\prime$ of these two tensors, consider the dot product of these two columns. We have
    \begin{align}
    &\sum\limits_{j_1,j_3}\left(\sum\limits_{j_2}P_{j_1 j_2}C_{j_3 j_2 k}\right)\left(\sum\limits_{j^\prime_2}P_{j_1 j^\prime_2}C^\prime_{j_3 j^\prime_2 k^\prime}\right) \nonumber\\
        = &\sum_{j_2,j^\prime_2}\left(\sum_{j_1} P_{j_1j_2}P_{j_1j^\prime_2}\right)\left(\sum\limits_{j_3} C_{j_3j_2k}C^\prime_{j_3j^\prime_2k^\prime}\right) \nonumber\\
        \overset{(a)}{=} &\left\lVert{P}\right\rVert^2_{col}\underbrace{\left(\sum\limits_{j_2,j_3} C_{j_3j_2k}C^\prime_{j_3j_2k^\prime}\right)}_{0} \nonumber\\
        \overset{(b)}{=} &0.
    \end{align}
    Equations (a) and (b) hold by Lemma \ref{lemma:ortho} in the general case. This general condition applies to the columns of path matrices of length $3$ as the initial condition. Now, we assume that the columns of path matrices of the same length are mutually orthogonal and further prove that when the length is increased by $1$, all path matrices continue to have orthogonal columns. Given two distinct path matrices $P$ and $P^\prime$ (the case of identical paths is covered above) and two CG tensors $C$ and $C^\prime$, we obtain
    
    \begin{align}
    &\sum\limits_{j_1,j_3}\left(\sum\limits_{j_2}P_{j_1 j_2}C_{j_3 j_2 k}\right)\left(\sum\limits_{j^\prime_2}P^\prime_{j_1 j^\prime_2}C^\prime_{j_3 j^\prime_2 k^\prime}\right) \nonumber\\
        = &\sum_{j_2,j^\prime_2}\underbrace{\left(\sum_{j_1} P_{j_1j_2}P^\prime_{j_1j^\prime_2}\right)}_0\left(\sum\limits_{j_3} C_{j_3j_2k}C^\prime_{j_3j^\prime_2k^\prime}\right) \nonumber\\
        = &0.
    \end{align}
    The final equality holds due to our assumption. Therefore, the columns of different path matrices of the same path length are orthogonal.
\end{proof}

\section{Complexity Analysis}
\label{sec:complexityanalysis}
For the complexity with respect to the dimension, our ICT decomposition has polynomial complexity, while \cite{bonvicini2024irreducible} has $\mathcal{O}\left((\log d)!\right)$ complexity, which is higher than any polynomial complexity. Here is the derivation. Let $d$ denotes the dimension and recall that Stirling's approximation tells us:

\begin{align}
    &(\log d)!\nonumber\\
    &\approx \left(\log d\right)^{\log d} \exp\left(-\log d\right) \sqrt{2\pi \log d} \nonumber\\
    &> \left(\log d\right)^{\log d} \exp\left(-\log d\right) \nonumber\\
    &=\exp\left(\log\left(\left(\log d\right)^{\log d}\right)-\log d\right)  \nonumber\\
    &=\exp\left({\log d}\log\left(\log d\right)-\log d\right).
\end{align}
Take the logarithm of this equation, we obtain

\begin{align}
\label{eq:log}
    &\log\left(\left(\log d\right)!\right) \nonumber\\
    &=\log d \log\left(\log d\right) - \log d \nonumber\\
    &= \log d\left(\log\left(\log d\right)-1\right).
\end{align}
On the other hand, taking the logarithm of an arbitrary polynomial $d^k$, where $k$ is constant, gives us

\begin{equation}
\label{eq:polynom}
    \log\left(d^k\right) = k\log d.
\end{equation}
Consequently, the ratio of equations (\ref{eq:log}) and (\ref{eq:polynom}) is

\begin{align}
    &\frac{\log\left(\left(\log d\right)!\right)}{\log\left(d^k\right)} \nonumber\\
    &=\frac{\log d\left(\log\left(\log d\right)-1\right)}{k\log d} \nonumber\\
    &=\frac{\left(\log\left(\log d\right)-1\right)}{k}.
\end{align}
As $d$ increases, $\frac{\log\left(\left(\log d\right)!\right)}{\log\left(d^k\right)}$ must be eventually higher than $1$, thus $\mathcal{O}\left(\left(\log d\right)!\right)$ grows faster than any polynomial.

\section{\changed{An Example for Generating Rank-$2$ ICT Decomposition Matrices}}
\label{sec:examplerank2}
\changed{Since we only focus on the rank-$2$ ICT decomposition, it suffices to consider the portion of Figure 2 with ranks and weights in $\{0,1,2\}$. As in Algorithm 1, we first initialize $(0)$-path matrices $P^{(p_0)}=C^{000}$, and flatten the first two indices. Maintain the path $p_0=(0)$, the current column number of the parentage scheme $i=0$, and the current weight $l_0=0$. Then, we iterate Step 2 for those available weights in the $(i+1)$-th column, i.e., weights $\abs{l_i-1}\leq l_{i+1}\leq l_i+1$ for $l_i\neq 0$. We then update the path by adding $l+1$ to the end of path $p_{i}$ to obtain $p_{i+1}$, and calculating the contraction,
        \begin{equation*}
            P^{(p_{i+1})} = \mathrm{flatten}_2(P^{(p_i)}\odot C^{1l_il_{i+1}}).
        \end{equation*}
Then, we update the column number $i \gets i+1$. Specifically, since we initialize that $i=0$, we have $i+1=1$, following the rule that weights $l_{i+1}$ must satisfy $\abs{l_i-1}\leq l_{i+1}\leq l_i+1$, the only possible $l_{i+1}$ is $1$. Thus, we have
\begin{equation*}
            P^{(0\to 1)} = {\rm flatten}_2(P^{(0)}\odot C^{101}).
\end{equation*}
Then, we have $l_1 =1$, and $l_2$ has three choices, $0$, $1$, and $2$. When $l_2 = 2$, we have
\begin{equation*}
            P^{(0\to 1\to 2)} = {\rm flatten}_2(P^{(0\to 1)}\odot C^{112});
\end{equation*}
When $l_2 = 1$, we have
\begin{equation*}
            P^{(0\to 1\to 1)} = {\rm flatten}_2(P^{(0\to 1)}\odot C^{111});
\end{equation*}
When $l_2 = 0$, we have
\begin{equation*}
            P^{(0\to 1\to 0)} = {\rm flatten}_2(P^{(0\to 1)}\odot C^{110}).
\end{equation*}
Then, we normalize the above three path matrices,
\begin{equation*}
            \hat{P}^{(0\to 1\to 2)}_{*,c}= \frac{P^{(0\to 1\to 2)}_{*,c}}{ \norm{P^{(0\to 1\to 2)}_{*,c}}_2}, \quad \hat{P}^{(0\to 1\to 1)}_{*,c}= \frac{P^{(0\to 1\to 1)}_{*,c}}{ \norm{P^{(0\to 1\to 1)}_{*,c}}_2}, \quad\hat{P}^{(0\to 1\to 0)}_{*,c}= \frac{P^{(0\to 1\to 0)}_{*,c}}{ \norm{P^{(0\to 1\to 0)}_{*,c}}_2}.
        \end{equation*}
The matrix multiplications then give us all three rank-$2$ ICT decomposition matrices,
\begin{equation*}
    H^{(2;2;1)} = \hat{P}^{(0\to 1\to 2)}\cdot (\hat{P}^{(0\to 1\to 2)})^{\top},
\end{equation*}
\begin{equation*}
    H^{(2;1;1)} = \hat{P}^{(0\to 1\to 1)}\cdot (\hat{P}^{(0\to 1\to 1)})^{\top},
\end{equation*}
\begin{equation*}
    H^{(2;0;1)} = \hat{P}^{(0\to 1\to 0)}\cdot (\hat{P}^{(0\to 1\to 0)})^{\top}.
\end{equation*}}

\newpage

\section{Visualization of Decomposition Matrices}
\label{sec:visualization}
\FloatBarrier

\begin{figure*}[h]

  \centering
  \includegraphics[width=0.7\linewidth]{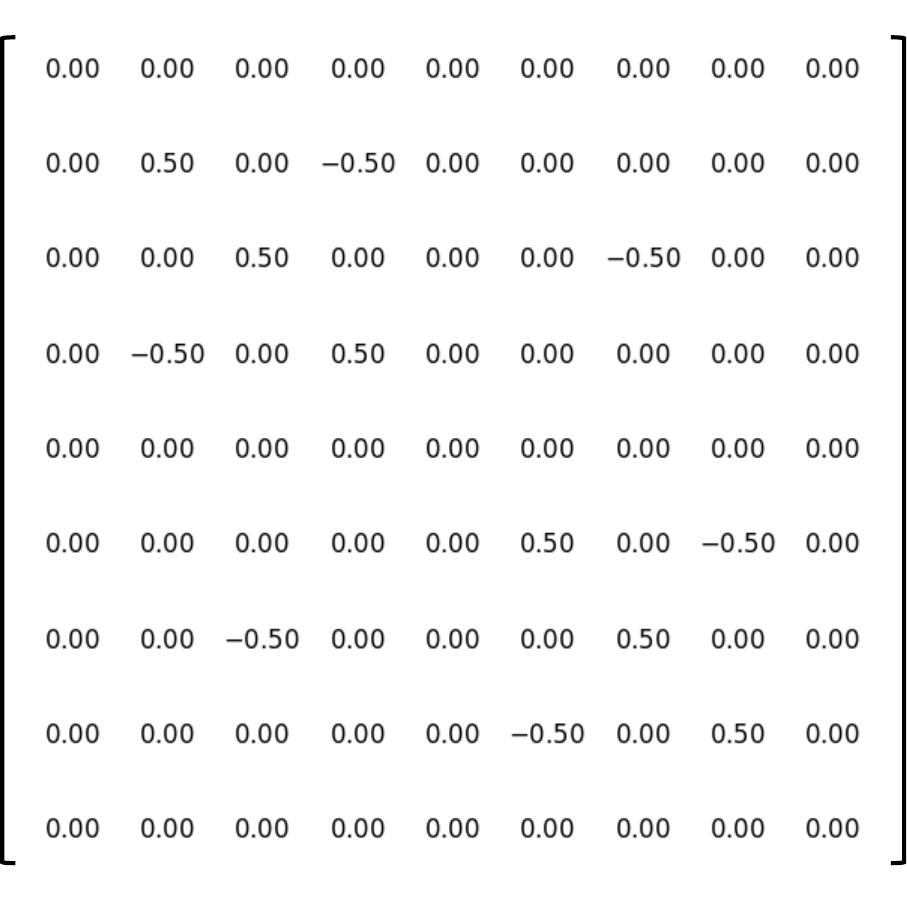}
  \caption{\changed{Decomposition matrix for rank-2 ICT with path $(0\to 1\to 1)$.}
  }
  \label{fig:011}
\vspace{-0.5cm}
\end{figure*}

\begin{figure*}[h]

  \centering
  \includegraphics[width=0.7\linewidth]{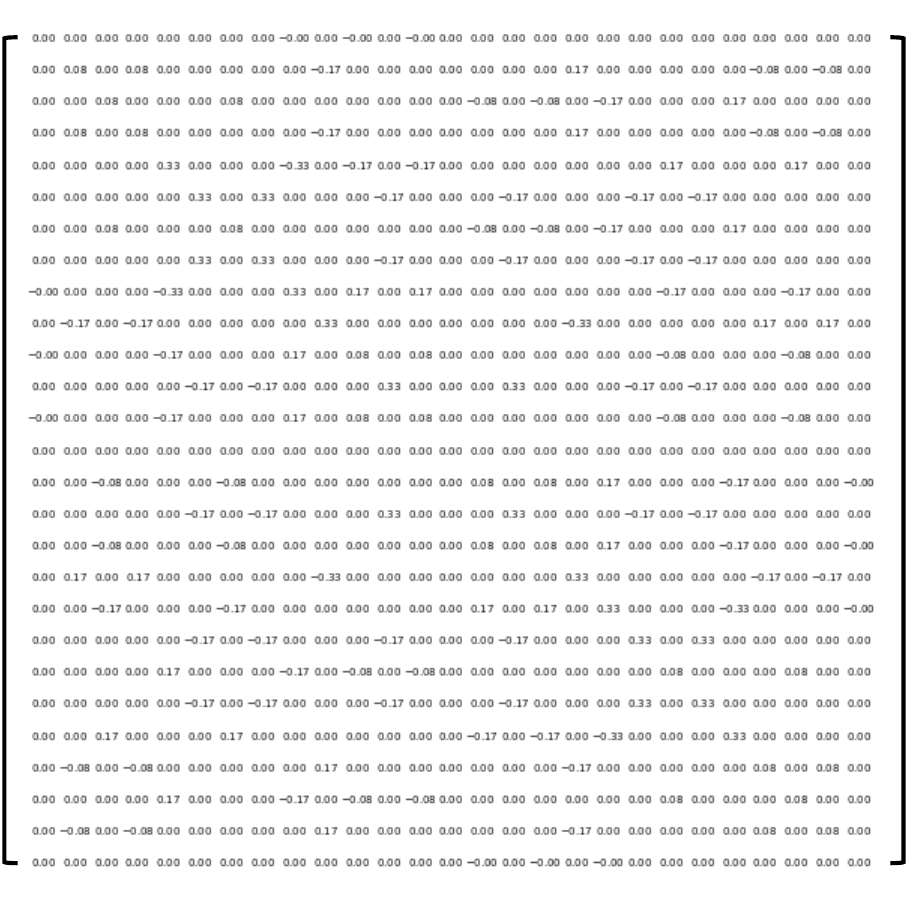}
  \caption{\changed{Decomposition matrix for rank-3 ICT with path $(0\to 1\to 2\to 2)$.}
  }
  \label{fig:0122}
\vspace{-0.5cm}
\end{figure*}

\begin{figure*}[h]

  \centering
  \includegraphics[width=0.7\linewidth]{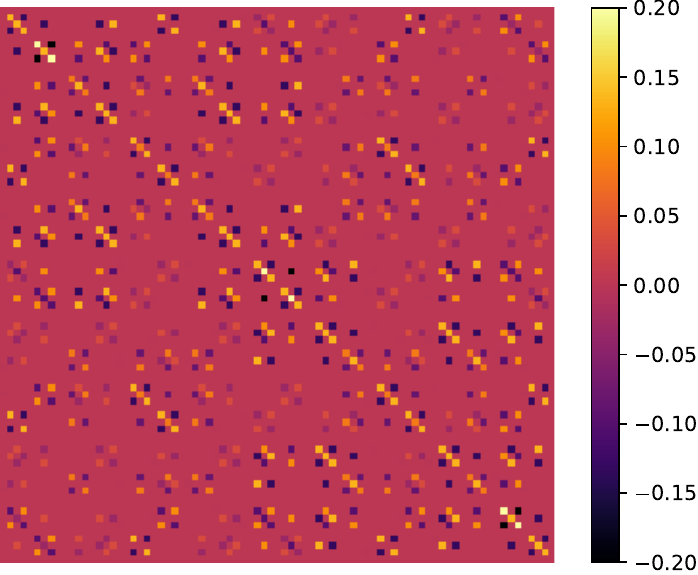}
  \caption{Decomposition matrix for \changed{rank-4} ICT with path $(\changed{0\to} 1\to 1\to 2 \to 3)$.
  }
  \label{fig:benchmarkchignolin}
\vspace{-0.5cm}
\end{figure*}

\begin{figure*}[h]

  \centering
  \includegraphics[width=0.7\linewidth]{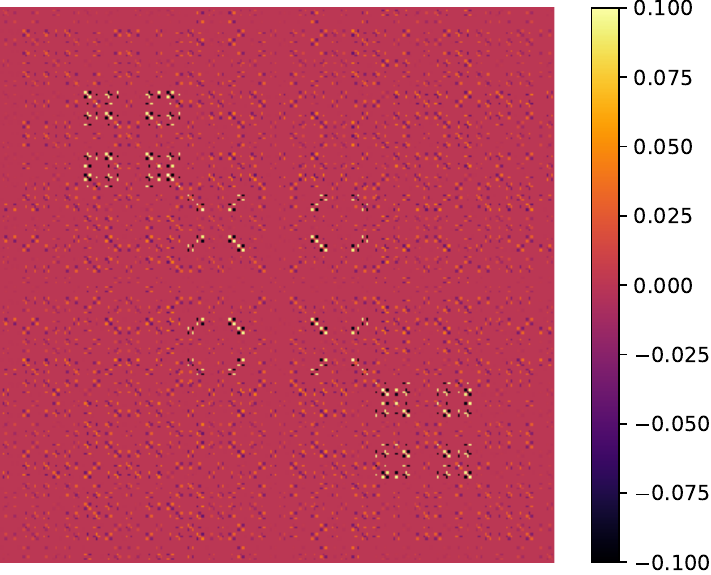}
  \caption{Decomposition matrix for \changed{rank-5} ICT with path $(\changed{0\to} 1\to 1\to 2 \to 3 \to 3)$.
  }
  \label{fig:benchmarkchignolin}

\end{figure*}

\begin{figure*}[h]

  \centering
  \includegraphics[width=0.7\linewidth]{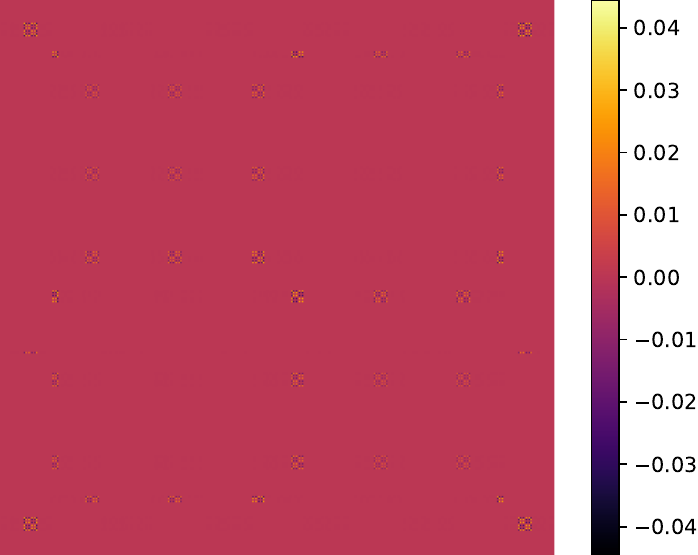}
  
  \caption{Decomposition matrix for \changed{rank-6} ICT with path $(\changed{0\to}1\to 1\to 0 \to 1 \to 2 \to 1)$.
  }
  \label{fig:benchmarkchignolin}

\end{figure*}

\begin{figure*}[h]

  \centering
  \includegraphics[width=0.7\linewidth]{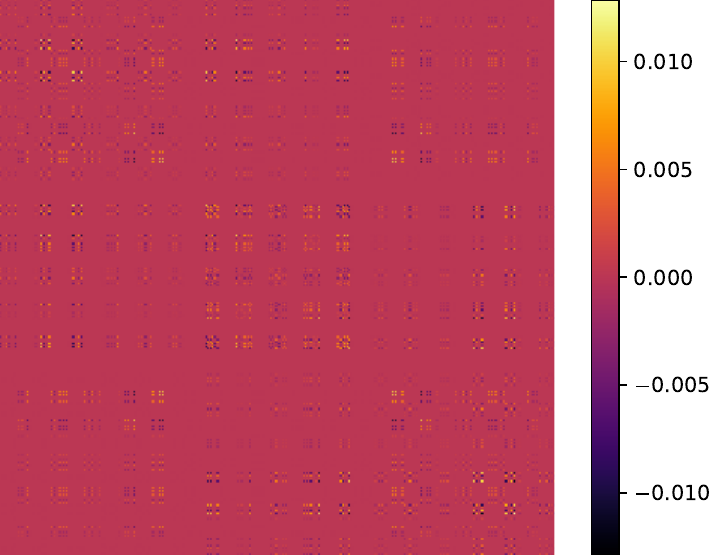}
  \caption{Decomposition matrix for \changed{rank-7} ICT with path $(\changed{0\to}1\to 2\to 3 \to 2 \to 2 \to 1 \to 1)$.
  }
  \label{fig:benchmarkchignolin}

\end{figure*}

\begin{figure*}[h]

  \centering
  \includegraphics[width=0.7\linewidth]{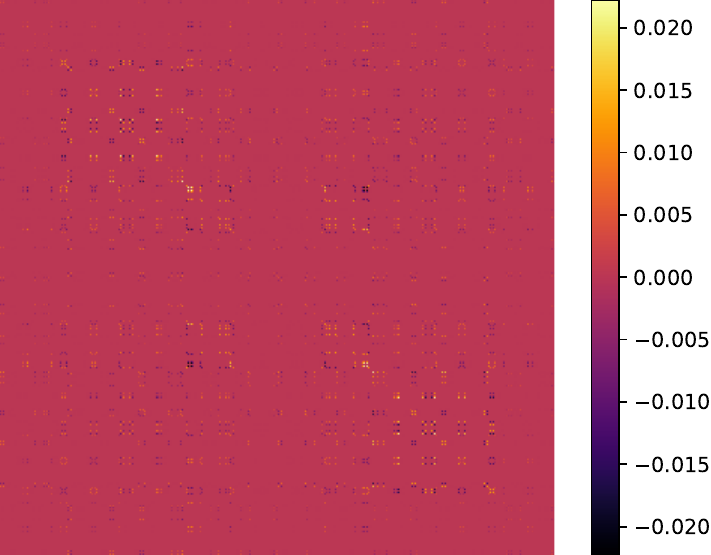}
  \caption{Decomposition matrix for \changed{rank-7} ICT with path $(\changed{0\to}1\to 1\to 1 \to 2 \to 1 \to 2 \to 2)$.
  }
  \label{fig:benchmarkchignolin}

\end{figure*}

\begin{figure*}[h]

  \centering
  \includegraphics[width=0.7\linewidth]{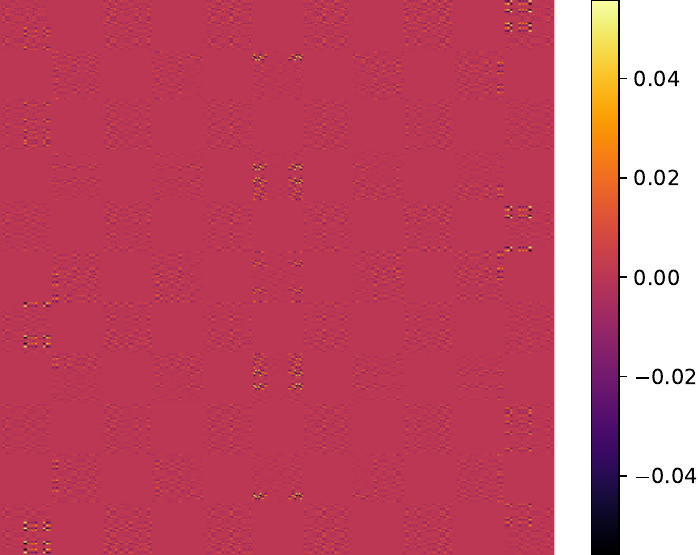}
  \caption{Equivariant basis matrix of ${\rm End}_{O(3)}\left(\left(\mathbbm{R}^{3}\right)^{\otimes 6}\right)$ with paths $(1\to 2\to 2 \to 3\to3\to3)$ and $(1\to 1\to 2 \to 3\to4\to3)$.
  }
  \label{fig:benchmarkchignolin}

\end{figure*}

\begin{figure*}[h]

  \centering
  \includegraphics[width=0.7\linewidth]{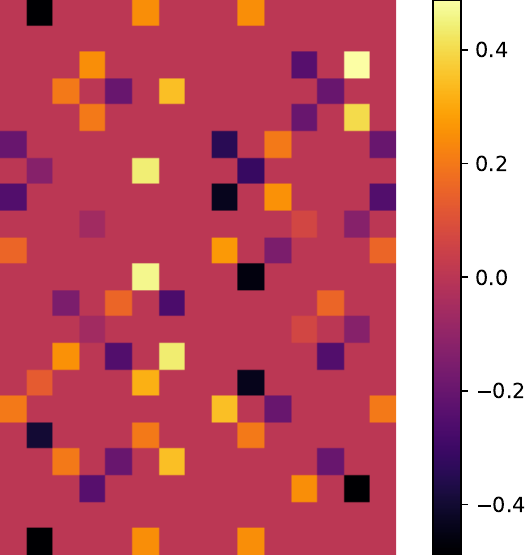}
  \caption{Equivariant basis matrix of ${\rm Hom}_{O(3)}\left(\left(l=1\otimes l=2\right),\left(l=1\otimes l=3\right)\right)$ with paths $(1\overset{2}{\to} 2)$ and $(1\overset{3}{\to} 2)$.
  }
  \label{fig:benchmarkchignolin}

\end{figure*}

\begin{figure*}[h]

  \centering
  \includegraphics[width=0.7\linewidth]{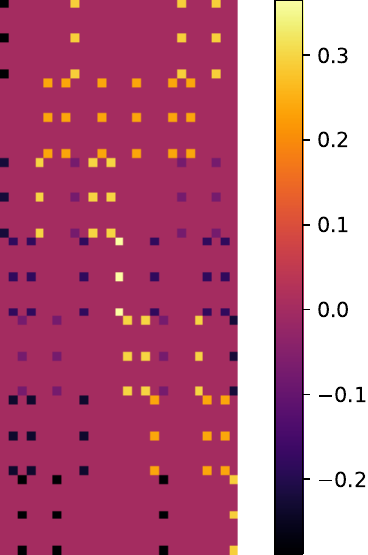}
  \caption{Equivariant basis matrix of 
  ${\rm Hom}_{SU(2)}\left(\left(0.5\otimes 0.5\otimes 1.5\right),\left(0.5\otimes 0.5\otimes 0.5\right)\right)$ with paths $(0.5\overset{0.5}{\to} 0\overset{1.5}{\to} 1.5)$ and $(0.5\overset{0.5}{\to} 1\overset{0.5}{\to} 1.5)$.
  }
  \label{fig:benchmarkchignolin}

\end{figure*}

\clearpage
\bibliography{sample}

\begin{thebibliography}{41}
\providecommand{\natexlab}[1]{#1}
\providecommand{\url}[1]{\texttt{#1}}
\expandafter\ifx\csname urlstyle\endcsname\relax
  \providecommand{\doi}[1]{doi: #1}\else
  \providecommand{\doi}{doi: \begingroup \urlstyle{rm}\Url}\fi

\bibitem[Alex et~al.(2011)Alex, Kalus, Huckleberry, and von Delft]{alex2011numerical}
Arne Alex, Matthias Kalus, Alan Huckleberry, and Jan von Delft.
\newblock A numerical algorithm for the explicit calculation of {$SU(N)$} and {$SL(N,C)$} {Clebsch--Gordan} coefficients.
\newblock \emph{Journal of Mathematical Physics}, 52\penalty0 (2):\penalty0 023507, 2011.

\bibitem[Andrews(2023)]{AndrewsSymmetry}
David~L. Andrews.
\newblock Symmetry-based identification and enumeration of independent tensor properties in nonlinear and chiral optics.
\newblock \emph{The Journal of Chemical Physics}, 158\penalty0 (3):\penalty0 034101, 2023.

\bibitem[Andrews and Ghoul(1982)]{andrews1982irreducible}
DL~Andrews and WA~Ghoul.
\newblock Irreducible fourth-rank {Cartesian} tensors.
\newblock \emph{Physical Review A}, 25\penalty0 (5):\penalty0 2647, 1982.

\bibitem[Batatia et~al.(2022)Batatia, Kovacs, Simm, Ortner, and Cs{\'a}nyi]{batatia2022mace}
Ilyes Batatia, David~P Kovacs, Gregor Simm, Christoph Ortner, and G{\'a}bor Cs{\'a}nyi.
\newblock {MACE}: Higher order equivariant message passing neural networks for fast and accurate force fields.
\newblock In \emph{Advances in Neural Information Processing Systems (NeurIPS)}, volume~35, pages 11423--11436, 2022.

\bibitem[Batzner et~al.(2022)Batzner, Musaelian, Sun, Geiger, Mailoa, Kornbluth, Molinari, Smidt, and Kozinsky]{batzner20223}
Simon Batzner, Albert Musaelian, Lixin Sun, Mario Geiger, Jonathan~P Mailoa, Mordechai Kornbluth, Nicola Molinari, Tess~E Smidt, and Boris Kozinsky.
\newblock E(3)-equivariant graph neural networks for data-efficient and accurate interatomic potentials.
\newblock \emph{Nature Communications}, 13\penalty0 (1):\penalty0 2453, 2022.

\bibitem[Bonvicini(2024)]{bonvicini2024irreducible}
Andrea Bonvicini.
\newblock Irreducible {Cartesian} tensor decomposition: A computational approach.
\newblock \emph{The Journal of Chemical Physics}, 160\penalty0 (22):\penalty0 224105, 2024.

\bibitem[Brauer(1937)]{brauer1937algebras}
Richard Brauer.
\newblock On algebras which are connected with the semisimple continuous groups.
\newblock \emph{Annals of Mathematics}, 38\penalty0 (4):\penalty0 857--872, 1937.

\bibitem[Br{\"o}cker and Tom~Dieck(2013)]{brocker2013representations}
Theodor Br{\"o}cker and Tammo Tom~Dieck.
\newblock \emph{Representations of compact Lie groups}, volume~98.
\newblock Springer Science \& Business Media, 2013.

\bibitem[Cheng(2024)]{cheng2024cartesian}
Bingqing Cheng.
\newblock Cartesian atomic cluster expansion for machine learning interatomic potentials.
\newblock \emph{npj Computational Materials}, 10\penalty0 (1):\penalty0 157, 2024.

\bibitem[Chmiela et~al.(2017)Chmiela, Tkatchenko, Sauceda, Poltavsky, Sch{\"u}tt, and M{\"u}ller]{chmiela2017machine}
Stefan Chmiela, Alexandre Tkatchenko, Huziel~E Sauceda, Igor Poltavsky, Kristof~T Sch{\"u}tt, and Klaus-Robert M{\"u}ller.
\newblock Machine learning of accurate energy-conserving molecular force fields.
\newblock \emph{Science Advances}, 3\penalty0 (5):\penalty0 e1603015, 2017.

\bibitem[Cohen and Welling(2016)]{cohen2016group}
Taco Cohen and Max Welling.
\newblock Group equivariant convolutional networks.
\newblock In \emph{International Conference on Machine Learning (ICML)}, pages 2990--2999, 2016.

\bibitem[Coope et~al.(1965)Coope, Snider, and McCourt]{coope1965irreducible}
JAR Coope, RF~Snider, and FR~McCourt.
\newblock Irreducible {Cartesian} tensors.
\newblock \emph{The Journal of Chemical Physics}, 43\penalty0 (7):\penalty0 2269--2275, 1965.

\bibitem[Din{\c{c}}kal(2013)]{dincckal2013orthonormal}
{\c{C}}i{\u{g}}dem Din{\c{c}}kal.
\newblock Orthonormal decomposition of third rank tensors and applications.
\newblock In \emph{Proceedings of the World Congress on Engineering}, volume~1, 2013.

\bibitem[Drautz(2019)]{drautz2019atomic}
Ralf Drautz.
\newblock Atomic cluster expansion for accurate and transferable interatomic potentials.
\newblock \emph{Physical Review B}, 99\penalty0 (1):\penalty0 014104, 2019.

\bibitem[Finzi et~al.(2021)Finzi, Welling, and Wilson]{finzi2021practical}
Marc Finzi, Max Welling, and Andrew~Gordon Wilson.
\newblock A practical method for constructing equivariant multilayer perceptrons for arbitrary matrix groups.
\newblock In \emph{International Conference on Machine Learning (ICML)}, pages 3318--3328, 2021.

\bibitem[Fu et~al.(2023)Fu, Wu, Wang, Xie, Keten, Gomez-Bombarelli, and Jaakkola]{fu2022forces}
Xiang Fu, Zhenghao Wu, Wujie Wang, Tian Xie, Sinan Keten, Rafael Gomez-Bombarelli, and Tommi Jaakkola.
\newblock Forces are not enough: Benchmark and critical evaluation for machine learning force fields with molecular simulations.
\newblock \emph{Transactions on Machine Learning Research (TMLR)}, 2023.

\bibitem[Gasteiger et~al.(2020)Gasteiger, Gro{\ss}, and G{\"u}nnemann]{gasteiger2020directional}
Johannes Gasteiger, Janek Gro{\ss}, and Stephan G{\"u}nnemann.
\newblock Directional message passing for molecular graphs.
\newblock In \emph{International Conference on Learning Representations (ICLR)}, 2020.

\bibitem[Gasteiger et~al.(2021)Gasteiger, Becker, and G{\"u}nnemann]{gasteiger2021gemnet}
Johannes Gasteiger, Florian Becker, and Stephan G{\"u}nnemann.
\newblock Gemnet: Universal directional graph neural networks for molecules.
\newblock In \emph{Advances in Neural Information Processing Systems (NeurIPS)}, volume~34, pages 6790--6802, 2021.

\bibitem[Gilmer et~al.(2017)Gilmer, Schoenholz, Riley, Vinyals, and Dahl]{gilmer2017mpnn}
Justin Gilmer, Samuel~S Schoenholz, Patrick~F Riley, Oriol Vinyals, and George~E Dahl.
\newblock Neural message passing for quantum chemistry.
\newblock In \emph{International Conference on Machine Learning (ICML)}, pages 1263--1272, 2017.

\bibitem[Gori et~al.(2005)Gori, Monfardini, and Scarselli]{gori2005new}
Marco Gori, Gabriele Monfardini, and Franco Scarselli.
\newblock A new model for learning in graph domains.
\newblock In \emph{International Joint Conference on Neural Networks (IJCNN)}, volume~2, pages 729--734, 2005.

\bibitem[Inui et~al.(2012)Inui, Tanabe, and Onodera]{inui2012group}
Teturo Inui, Yukito Tanabe, and Yositaka Onodera.
\newblock \emph{Group theory and its applications in physics}, volume~78.
\newblock Springer Science \& Business Media, 2012.

\bibitem[Kondor et~al.(2018)Kondor, Lin, and Trivedi]{kondor2018clebsch}
Risi Kondor, Zhen Lin, and Shubhendu Trivedi.
\newblock Clebsch--{G}ordan nets: a fully {F}ourier space spherical convolutional neural network.
\newblock \emph{Advances in Neural Information Processing Systems (NeurIPS)}, 31, 2018.

\bibitem[LeCun et~al.(1989)LeCun, Boser, Denker, Henderson, Howard, Hubbard, and Jackel]{lecun1989backpropagation}
Yann LeCun, Bernhard Boser, John~S Denker, Donnie Henderson, Richard~E Howard, Wayne Hubbard, and Lawrence~D Jackel.
\newblock Backpropagation applied to handwritten zip code recognition.
\newblock \emph{Neural Computation}, 1\penalty0 (4):\penalty0 541--551, 1989.

\bibitem[Liao et~al.(2023)Liao, Wood, Das, and Smidt]{liao2023equiformerv2}
Yi-Lun Liao, Brandon~M Wood, Abhishek Das, and Tess Smidt.
\newblock Equiformerv2: Improved equivariant transformer for scaling to higher-degree representations.
\newblock In \emph{International Conference on Learning Representations (ICLR)}, 2023.

\bibitem[Marsaglia and PH~Styan(1974)]{matsaglia1974equalities}
George Marsaglia and George PH~Styan.
\newblock Equalities and inequalities for ranks of matrices.
\newblock \emph{Linear and Multilinear Algebra}, 2\penalty0 (3):\penalty0 269--292, 1974.

\bibitem[Mihailov(1977)]{momentaMihailov_1977}
V~V Mihailov.
\newblock Addition or arbitrary number of identical angular momenta.
\newblock \emph{Journal of Physics A: Mathematical and General}, 10\penalty0 (2):\penalty0 147, 1977.

\bibitem[Musaelian et~al.(2023)Musaelian, Batzner, Johansson, Sun, Owen, Kornbluth, and Kozinsky]{musaelian2023learning}
Albert Musaelian, Simon Batzner, Anders Johansson, Lixin Sun, Cameron~J Owen, Mordechai Kornbluth, and Boris Kozinsky.
\newblock Learning local equivariant representations for large-scale atomistic dynamics.
\newblock \emph{Nature Communications}, 14\penalty0 (1):\penalty0 579, 2023.

\bibitem[Pearce-Crump(2023{\natexlab{a}})]{pearce2023algorithm}
Edward Pearce-Crump.
\newblock An algorithm for computing with {B}rauer's group equivariant neural network layers.
\newblock \emph{arXiv preprint arXiv:2304.14165}, 2023{\natexlab{a}}.

\bibitem[Pearce-Crump(2023{\natexlab{b}})]{pearce2023brauer}
Edward Pearce-Crump.
\newblock Brauer’s group equivariant neural networks.
\newblock In \emph{International Conference on Machine Learning (ICML)}, pages 27461--27482, 2023{\natexlab{b}}.

\bibitem[Racah(1942)]{theoryofcomplexspectra}
Giulio Racah.
\newblock Theory of complex spectra. {II}.
\newblock \emph{Physical Review}, 62:\penalty0 438--462, 1942.

\bibitem[Ruddigkeit et~al.(2012)Ruddigkeit, Van~Deursen, Blum, and Reymond]{ruddigkeit2012enumeration}
Lars Ruddigkeit, Ruud Van~Deursen, Lorenz~C Blum, and Jean-Louis Reymond.
\newblock Enumeration of 166 billion organic small molecules in the chemical universe database gdb-17.
\newblock \emph{Journal of Chemical Information and Modeling}, 52\penalty0 (11):\penalty0 2864--2875, 2012.

\bibitem[Sch{\"u}tt et~al.(2017)Sch{\"u}tt, Kindermans, Sauceda~Felix, Chmiela, Tkatchenko, and M{\"u}ller]{schutt2017schnet}
Kristof Sch{\"u}tt, Pieter-Jan Kindermans, Huziel~Enoc Sauceda~Felix, Stefan Chmiela, Alexandre Tkatchenko, and Klaus-Robert M{\"u}ller.
\newblock Schnet: A continuous-filter convolutional neural network for modeling quantum interactions.
\newblock In \emph{Advances in Neural Information Processing Systems (NeurIPS)}, volume~30, 2017.

\bibitem[Sch{\"u}tt et~al.(2021)Sch{\"u}tt, Unke, and Gastegger]{schutt2021equivariant}
Kristof Sch{\"u}tt, Oliver Unke, and Michael Gastegger.
\newblock Equivariant message passing for the prediction of tensorial properties and molecular spectra.
\newblock In \emph{International Conference on Machine Learning (ICML)}, pages 9377--9388, 2021.

\bibitem[Simeon and De~Fabritiis(2024)]{simeon2024tensornet}
Guillem Simeon and Gianni De~Fabritiis.
\newblock Tensornet: Cartesian tensor representations for efficient learning of molecular potentials.
\newblock In \emph{Advances in Neural Information Processing Systems (NeurIPS)}, volume~36, 2024.

\bibitem[Simeon et~al.(2024)Simeon, Mirarchi, Pelaez, Galvelis, and De~Fabritiis]{simeon2024inclusion}
Guillem Simeon, Antonio Mirarchi, Raul~P Pelaez, Raimondas Galvelis, and Gianni De~Fabritiis.
\newblock On the inclusion of charge and spin states in {C}artesian tensor neural network potentials.
\newblock \emph{arXiv preprint arXiv:2403.15073}, 2024.

\bibitem[Snider(2017)]{snider2017irreducible}
Robert~F Snider.
\newblock \emph{Irreducible {Cartesian} Tensors}, volume~43.
\newblock Walter de Gruyter GmbH \& Co KG, 2017.

\bibitem[Thomas et~al.(2018)Thomas, Smidt, Kearnes, Yang, Li, Kohlhoff, and Riley]{thomas2018tfn}
Nathaniel Thomas, Tess Smidt, Steven Kearnes, Lusann Yang, Li~Li, Kai Kohlhoff, and Patrick Riley.
\newblock Tensor field networks: Rotation-and translation-equivariant neural networks for 3{D} point clouds.
\newblock \emph{arXiv preprint arXiv:1802.08219}, 2018.

\bibitem[Wang et~al.(2024)Wang, Hart, Surovik, Kelestemur, Huang, Zhao, Yeatman, Wang, Walters, and Platt]{wangequivariant}
Dian Wang, Stephen Hart, David Surovik, Tarik Kelestemur, Haojie Huang, Haibo Zhao, Mark Yeatman, Jiuguang Wang, Robin Walters, and Robert Platt.
\newblock Equivariant diffusion policy.
\newblock In \emph{The Conference on Robot Learning (CoRL)}, 2024.

\bibitem[Weiler et~al.(2018)Weiler, Geiger, Welling, Boomsma, and Cohen]{weiler20183d}
Maurice Weiler, Mario Geiger, Max Welling, Wouter Boomsma, and Taco~S Cohen.
\newblock 3{D} steerable {CNN}s: Learning rotationally equivariant features in volumetric data.
\newblock \emph{Advances in Neural Information Processing Systems (NeurIPS)}, 31, 2018.

\bibitem[Zaverkin et~al.(2024)Zaverkin, Alesiani, Maruyama, Errica, Christiansen, Takamoto, Weber, and Niepert]{zaverkin2024higher}
Viktor Zaverkin, Francesco Alesiani, Takashi Maruyama, Federico Errica, Henrik Christiansen, Makoto Takamoto, Nicolas Weber, and Mathias Niepert.
\newblock Higher-rank irreducible cartesian tensors for equivariant message passing.
\newblock \emph{arXiv preprint arXiv:2405.14253}, 2024.

\bibitem[Zou et~al.(2001)Zou, Zheng, Du, and Rychlewski]{zou2001orthogonal}
W-N Zou, Q-S Zheng, D-X Du, and J~Rychlewski.
\newblock Orthogonal irreducible decompositions of tensors of high orders.
\newblock \emph{Mathematics and Mechanics of Solids}, 6\penalty0 (3):\penalty0 249--267, 2001.

\end{thebibliography}

\end{document}